\documentclass[twoside,11pt]{article}

\usepackage{jmlr2e}

\usepackage{microtype}
\usepackage[T1]{fontenc}    
\usepackage{booktabs}       
\usepackage{amsfonts}       
\usepackage{nicefrac}       
\usepackage{algorithm}
\usepackage{algorithmic}
\usepackage{graphicx}
\usepackage{subfigure} 
\usepackage{bm}
\usepackage{amssymb}
\usepackage{enumitem}
\usepackage{mathtools}
\usepackage{array}
\usepackage{lipsum}
\usepackage{multirow}
\usepackage{makecell}
\usepackage{xcolor}
\usepackage{tikz}
\usepackage{etoolbox}
\newcommand*{\circled}[1]{\lower.7ex\hbox{\tikz\draw (0pt, 0pt)%
    circle (.5em) node {\makebox[1em][c]{\small #1}};}}
\robustify{\circled}

\usepackage{times}
\usepackage{soul}
\usepackage{url}
\usepackage[utf8]{inputenc}
\usepackage[small]{caption}
\usepackage{graphicx}
\usepackage{amsmath}
\usepackage{booktabs}
\usepackage{algorithm}
\usepackage{algorithmic}

\usepackage{afterpage}
\usepackage{makecell}
\usepackage{tabulary}
\usepackage{xcolor}
\usepackage{colortbl}
\usepackage{prettyref}
\usepackage{framed}
\usepackage{booktabs}       
\newcommand{\pref}[1]{\prettyref{#1}}

\newcommand{\savehyperref}[2]{\texorpdfstring{\hyperref[#1]{#2}}{#2}}
\newrefformat{eq}{\savehyperref{#1}{Eq. \textup{(\ref*{#1})}}}
\newrefformat{eqn}{\savehyperref{#1}{Eq.~\textup{(\ref*{#1})}}}
\newrefformat{lem}{\savehyperref{#1}{\textbf{Lemma~\ref*{#1}}}}
\newrefformat{assump}{\savehyperref{#1}{\textbf{Assumption~\ref*{#1}}}}
\newrefformat{defi}{\savehyperref{#1}{\textbf{Definition~\ref*{#1}}}}
\newrefformat{tab}{\savehyperref{#1}{Table~\ref*{#1}}}
\newrefformat{lemma}{\savehyperref{#1}{Lemma~\ref*{#1}}}
\newrefformat{line}{\savehyperref{#1}{Line~\ref*{#1}}}
\newrefformat{thm}{\savehyperref{#1}{\textbf{Theorem~\ref*{#1}}}}
\newrefformat{corr}{\savehyperref{#1}{\textbf{Corollary~\ref*{#1}}}}
\newrefformat{cor}{\savehyperref{#1}{\textbf{Corollary~\ref*{#1}}}}
\newrefformat{sec}{\savehyperref{#1}{Section~\ref*{#1}}}
\newrefformat{app}{\savehyperref{#1}{Appendix~\ref*{#1}}}
\newrefformat{ex}{\savehyperref{#1}{Example~\ref*{#1}}}
\newrefformat{fig}{\savehyperref{#1}{Figure~\ref*{#1}}}
\newrefformat{alg}{\savehyperref{#1}{Algorithm~\ref*{#1}}}
\newrefformat{rem}{\savehyperref{#1}{Remark~\ref*{#1}}}
\newrefformat{conj}{\savehyperref{#1}{Conjecture~\ref*{#1}}}
\newrefformat{prop}{\savehyperref{#1}{Proposition~\ref*{#1}}}
\newrefformat{proto}{\savehyperref{#1}{Protocol~\ref*{#1}}}
\newrefformat{prob}{\savehyperref{#1}{Problem~\ref*{#1}}}
\newrefformat{claim}{\savehyperref{#1}{Claim~\ref*{#1}}}
\newrefformat{que}{\savehyperref{#1}{Question~\ref*{#1}}}
\newrefformat{op}{\savehyperref{#1}{Open Problem~\ref*{#1}}}
\newrefformat{fn}{\savehyperref{#1}{Footnote~\ref*{#1}}}
\newrefformat{property}{\savehyperref{#1}{Property~\ref*{#1}}}

\usepackage{xspace}
\usepackage{amsmath}
\usepackage{bbm}
\usepackage{algorithm}
\usepackage[algo2e, vlined, ruled]{algorithm2e}
\usepackage{algorithmic}
\usepackage{mathrsfs}
\usepackage{bm}
\usepackage{makecell}
\usepackage{tabulary}

\newcommand{\risk}{\text{\rm Risk}}

\newcommand{\diff}{\text{\rm Diff}}

\newcommand{\advrisk}{\text{\rm AdvRisk}}

\newcommand{\distribution}{P}
\newcommand{\distance}{d}

\newcommand{\calC}{\mathcal{C}}

\newcommand{\calD}{\mathcal{D}}
\newcommand{\calL}{\mathcal{L}}
\newcommand{\calN}{\mathcal{N}}

\newcommand{\calB}{\mathcal{B}}
\newcommand{\calE}{\mathcal{E}}

\newcommand{\calH}{\mathcal{H}}
\newcommand{\calS}{\mathcal{S}}

\newcommand{\calF}{\mathcal{F}}

\newcommand{\R}{\mathbb{R}}

\newcommand{\E}{\mathbb{E}}

\newcommand{\test}{\text{\rm test}}

\newcommand{\wtilde}{\widetilde}
\newcommand{\what}{\widehat}
\newcommand{\calP}{\mathcal{P}}

\newcommand{\calX}{\mathcal{X}}
\newcommand{\calY}{\mathcal{Y}}

\newcommand{\truextest}{X_{\test}}
\newcommand{\trueytest}{c(X_{\test})}

\newcommand{\truepara}{X}

\newcommand{\distmodel}{\widetilde{g}}
\newcommand{\truemodel}{g}
\newcommand{\protector}{\text{\rm prot}}
\newcommand{\attacker}{\text{\rm attack}}
\newcommand{\distpara}{\widetilde{X}}

\newcommand{\distdistribution}{\widetilde{P}}

\newcommand{\red}[1]{{\color{red}#1}}
\newcommand{\blue}[1]{{\color{blue}#1}}

\newtheorem{assumption}{Assumption}[section]

\newtheorem{defi}{Definition}[section]

\newtheorem{thm}{Theorem}[section]

\newtheorem{lem}[thm]{Lemma}

\usepackage{lastpage}

\ShortHeadings{Probably Approximately Correct Federated Learning}{}
\firstpageno{1}

\begin{document}

\title{Probably Approximately Correct Federated Learning}

\author{\name Xiaojin Zhang \email xiaojinzhang@ust.hk \\
\addr Hong Kong University of Science and Technology, China
      \AND 
      \name Anbu Huang \email stevenhuang@webank.com\\
      \addr Webank, Shenzhen, China
      \AND
      \name Lixin Fan \email lixinfan@webank.com\\
      \addr Webank, Shenzhen, China
      \AND
      \name Kai Chen \email kaichen@cse.ust.hk\\
      \addr Hong Kong University of Science and Technology, China
      \AND
      \name Qiang Yang\thanks{Corresponding author} \email qyang@cse.ust.hk \\
      \addr WeBank and Hong Kong University of Science and Technology, China
      }

\editor{}

\newcommand{\colnote}[3]{\textcolor{#1}{$\ll$\textsf{#2}$\gg$\marginpar{\tiny\bf #3}}}
\newcommand{\lfan}[1]{\colnote{red}{#1--Lixin}{LF}}

\maketitle


\begin{abstract}
Federated learning (FL) is a distributed learning framework with privacy, utility, and efficiency as its primary concerns. Existing research indicates that it is unlikely to simultaneously achieve privacy protection, utility and efficiency. How to find an optimal trade-off for the three factors is a key consideration for trustworthy federated learning.  Is it possible to cast the trade-off as a multi-objective optimization problem in order to minimize the utility loss and efficiency reduction while constraining the privacy leakage? Existing multi-objective optimization frameworks are not suitable for this goal because they are  time-consuming and provide no guarantee for the existence of the Pareto frontier for solutions.  This challenge motivates us to seek an alternative solution to transform the multi-objective problem into a single-objective problem, which can potentially be more cost-effective and optimal. To this end, we present FedPAC, a unified framework that leverages PAC learning to quantify multiple objectives in terms of sample complexity. This unification provides a quantification that allows us to constrain the solution space of multiple objectives to a shared dimension.  We can then solve the problem with the help of a single-objective optimization algorithm. Specifically, we provide the theoretical results and detailed analyses on how to quantify the utility loss, privacy leakage, and efficiency, and show that their trade-off can be attained at the maximal cost on potential attackers from the PAC learning perspective.
\end{abstract}

\begin{keywords}
Federated learning, PAC learning, Multi-objective optimization, Trade-off,  Generalization error bounds
\end{keywords}

\section{Introduction}

\medskip

Federated learning (FL) is a trustworthy machine learning framework,  where \textit{privacy}, \textit{utility}, and \textit{efficiency}
are the three main concerns that must be scrutinized and optimized. Existing research has shown that there always exists a trade-off among privacy protection, utility and efficiency (\cite{zhang2019theoretically}).  How to find an optimal trade-off among all desired objectives is a key consideration when designing and deploying the FL framework in practice (\cite{tsipras2018robustness}).

Most of the existing works cast the above-mentioned trade-off problem as a constrained optimization problem, which aims to minimize certain objectives subject to imposed constraints. For example, \citep{makhdoumi2013privacy,du2012privacy} formulated the privacy-utility trade-off as an optimization problem where {private information leaked to} the adversary is minimized {subject to} a given set of utility constraints. \citep{rassouli2019optimal} illustrated that the optimal privacy-utility trade-off can be solved using a standard linear program and provided a closed-form solution for the special case when the data to be released is a binary variable. {In essence, this line of research can be generalized to the following multi-objective constrained optimization problem:}

{
\begin{align} \label{eq: constraint_optimization_problem_add_max}
\begin{array}{r@{\quad}l@{}l@{\quad}l}
\quad\min &  (\epsilon_1,\ \dots,\ \epsilon_k) \\
\text{subject to} &  \epsilon_{k+j}\le\alpha_j,\ \ \ \ \forall \  j \in \{1,\cdots,m\},
\end{array}
\end{align}}

{In the above, $(\epsilon_1, ..., \epsilon_k)$ denote objectives to be optimized in federated learning and $(\epsilon_{k+1}, ..., \epsilon_{k+m})$ are constrains to be fulfilled.} The constrained multi-objective optimization as such poses several challenges that do not admit trivial solutions. First, one must prescribe in advance a proper constraint (i.e., $\alpha_j$) that ensures feasible solutions of equation \ref{eq: constraint_optimization_problem_add_max} to be useful in practice. Unfortunately, existing multi-objective optimization frameworks do not guarantee the existence of Pareto solutions that
minimizes all objective functions simultaneously (\cite{miettinen1999nonlinear,hassanzadeh2010multi}); Second, in most cases, the objective function of equation \ref{eq: constraint_optimization_problem_add_max} is usually not convex, direct multi-objective optimization cannot make use of efficient gradient-based approaches, and are often time-consuming by {using e.g. evolution-based methods ( \cite{DBLP:journals/csur/TianSZCHTJ22,DBLP:journals/ieeejas/TianCMZTJ22,DBLP:journals/memetic/ChengDDJ22,DBLP:journals/corr/abs-2210-08295,DBLP:journals/ieeejas/HuaLHJ21});} Third, the problem of equation \ref{eq: constraint_optimization_problem_add_max} can be scalarized into the single objective
optimization problem (i.e., $\lambda_1\epsilon_1+\dots+\lambda_k\epsilon_k$), {yet it remains an open problem of} how to set proper scalarizing coefficients such that optimal solutions can reach every Pareto optimal solutions of the original problem.  Furthermore, when considering how to model the efforts of potential attackers in the FL framework, we need to unify both the protection and the attacking mechanisms.  If we can find a unifying dimension in which to compare these efforts, it would be much easier to optimize the trustworthy FL framework under all three objectives.

{

Taking consideration of the aforementioned challenges of multi-objective problems, we propose a new unified framework that leverages the PAC learning theory (\cite{valiant1984theory}) to quantify multiple objectives in terms of the \textit{sample complexity} (denoted as $s$). The sample complexity is used in the PAC learning framework to measure the number of samples needed for a learning model to attain a certain level of competence within a given confidence level.  This number can be used to measure both the costs for the defender and the attacker in an FL framework.  With this in mind, we can reformulate the optimal trade-off problem as equation \ref{eq: fedpac_multi_2_single}.  
{\begin{align} \label{eq: fedpac_multi_2_single}
\begin{array}{r@{\quad}l@{}l@{\quad}l}
\quad\mathop{\min}\limits_{s} &  f(\epsilon_1(s),\epsilon_2(s),\dots,\epsilon_k(s)) \\
\text{subject to} & \epsilon_{i}\le f_i(s),\ \ \ \ \ \ \ \forall \ i \in \{1,\dots,k,\dots,(k+m)\} . 
\end{array}
\end{align}}
}

Here, $f$ is a specific function with sample complexity ($s$) as its variables.  We list the advantages of equation \ref{eq: fedpac_multi_2_single} as follows:

\begin{enumerate}
    \item Solutions for multiple objectives of equation \ref{eq: fedpac_multi_2_single} share the same space (i.e., sample complexity), while the solution spaces of equation \ref{eq: constraint_optimization_problem_add_max} are independent of each other. A shared and small dimensional space means that it is more likely to achieve the optimal solution instead of a suboptimal one.

    \item Following up the aforementioned conclusion, theoretically, it can potentially be more efficient to find the optimal solution in a smaller search space.

    

    \item Sample complexity is easier to understand and applied to model both the attackers' costs and protection expenses, thus leading to a trade-off mechanism for both the attackers and the defenders.
\end{enumerate}

The main contributions of our paper are summarized as follows:

\begin{itemize}
    \item We present a unified framework to analyze the federated learning algorithm. Specifically, we proposed FedPAC framework, to measure and quantify privacy leakage (\pref{defi:Privacy_Leakage}), utility loss (\pref{defi:generalization_error_pd}), and training efficiency (\pref{defi:Training_Efficiency}). 

     \item From the PAC learning perspective, we provide the upper bound for the utility loss (\pref{thm: relation_utility_loss}) and privacy leakage (\pref{thm: privacy_distortion_and_datasize}), and further formalize the privacy-utility-efficiency trade-off expression (\pref{thm: trade-off analysis for general protection mechanism}). Based on these results, we formalize the concept of private PAC learning (\pref{corr: utility_privacy_efficiency}), which serves as the basis for the algorithmic proposal of a novel protection mechanism.
     


     \item For the attacker, we provide a detailed analysis about the upper bound of the privacy leakage (Section \ref{subsec:upper_bound_privacy_leakage}), and analyze the cost of the attacker from the perspective of PAC learning (Section \ref{sec:PAC_Learnability_of_the_Attackers}).
     
    \item For the protector, we provide a detailed analysis about the upper bound of the utility loss (Section \ref{sec:Generalized_error_Bound}). Specifically, we analyze the utility loss caused by different protection mechanisms separately, including \textit{Randomization} (\cite{geyer2017differentially,truex2020ldp,abadi2016deep}) and \textit{Homomorphic Encryption (HE)} (\cite{gentry2009fully,batchCryp}) (Section \pref{subsec:trade_off_random_he}).


\end{itemize}

Our paper is organized as follows: In section \ref{sec:Preliminary}, we first provide necessary preliminary to better understand this paper, including learning scenario setting in section \ref{subsec:learning_setting}, threat model in section \ref{subsec:thread_model}, and summarize our main results in section \ref{subsec:main_results}. In section \ref{sec:related_work}, we review some related works of our paper, including privacy-utility-efficiency trade-off, multi-objective optimization, and PAC learning; In section \ref{sec:on_the_learnability_of_the_attacker},  we focus on the attacker, we provide a detailed analysis of the upper bound of privacy leakage and the cost of the attacker. In section \ref{sec:utility_loss_analysis}, we focus on the protector and provide a detailed analysis of the upper bound of utility loss in the presence of different protection mechanisms. In section \ref{sec:discussions}, we discuss how our results provide theoretical bounds on the gains and losses of privacy, utility, and efficiency in federated learning.


\section{Preliminary}
\label{sec:Preliminary}

In this section, we describe the learning scenario setting and threat model of our paper, as well as briefly summarize the main results.

\subsection{Learning scenario setting}
\label{subsec:learning_setting}

{
In this work, we consider the learning scenario based on horizontal federated learning (HFL), as shown in Figure \ref{fig:hfl}. 
\begin{figure}[!ht]
    \centering
    \includegraphics[width=0.7\linewidth]{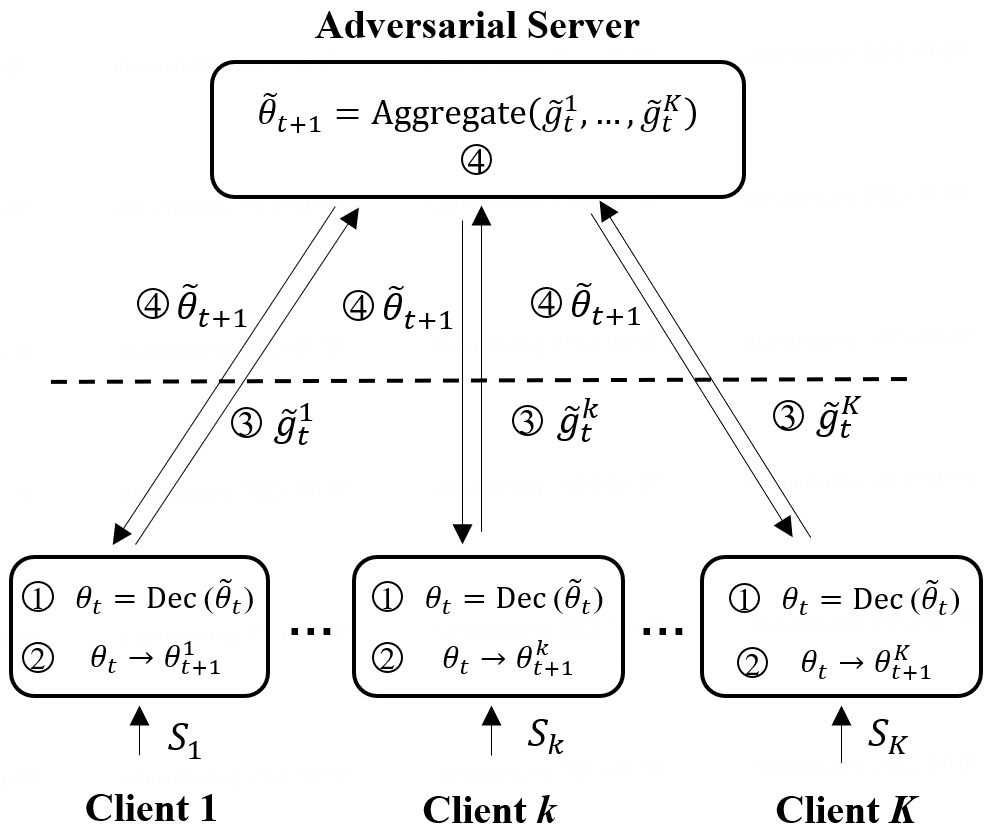}    \caption{Learning scenario of our work.}
    \label{fig:hfl}
\end{figure}

We assume there are a total of $K$ clients collaboratively to learn a shared model, the server side is an adversary attacker, and each client acts as a protector to prevent privacy leakage. For convenience, we first provide the notation used throughout this article, as shown in table \ref{table: notation}.

We follow the FL training procedure of \textit{federated averaging} (FedAvg) (\cite{mcmahan2017communication}), plus adding protection mechanisms for the sake of security concerns. Without loss of generality, suppose the training process is running on the $t$-th round iteration, let $\theta_t$ be the aggregated global model at round $t$, $\eta_t$ represent the learning rate, and $\delta_{t}^{(k)}$ be the distortion applied to protect the local information of client $k$ at iteration $t$. We denote $\calD_{\protector}^{(k)} = \{\textbf{X}^{(k)}, \textbf{Y}^{(k)}\} = \{(X_1^{(k)},Y_1^{(k)}),\cdots, (X_{m_k}^{(k)},Y_{m_k}^{(k)})\}$ as the \textit{local training set} of client $k$, and $\calL$ represents the loss function. The training process for privacy-preserving HFL contains the following four steps:

\begin{itemize}
    \item Upon receiving the global model $\wtilde{\theta}_{t}$ at round $t$, each protector decodes the global model $\theta_t\leftarrow \text{Dec}(\wtilde{\theta}_t)$.

    Please note that, some protection mechanisms, e.g., \textit{Randomization} (\cite{geyer2017differentially,truex2020ldp,abadi2016deep}), do not have the decoding procedure. However, for consistency, we keep this step and set $\text{Dec}(\wtilde \theta_t) = \wtilde \theta_t$ for those who without decoding operation.

    \item Take client $k$ as an example, with the decoded global model, the protector $k$ updates its local model parameters using SGD as: 
    \begin{align}\label{eq: hfl_step2}
       \theta_{t+1}^{(k)} \leftarrow \theta_{t} -\eta_t\cdot\frac{1}{|\calD_{\protector}^{(k)}|}\sum_{i = 1}^{|\calD_{\protector}^{(k)}|} \nabla\calL(\theta_t, X_i^{(k)}, Y_{i}^{(k)})
    \end{align}
    
    \item The protector $k$ uploads the distorted gradient to the server, which is defined as:
    \begin{align}\label{eq: distortion_approach}
       \wtilde g_{t}^{(k)}\leftarrow \frac{1}{|\calD_{\protector}^{(k)}|}\sum_{i = 1}^{|\calD_{\protector}^{(k)}|} \nabla\calL(\theta_t, X_i^{(k)}, Y_{i}^{(k)}) + \delta_{t}^{(k)},
    \end{align}
    where $\delta_{t}^{(k)}$ represents the distortion added by client $k$ at round $t$;
    \item Upon receiving distorted gradients from all clients, the server aggregates these information and updates global model as:
\begin{align}\label{eq: hfl_step4}
    \wtilde\theta_{t+1}\leftarrow \wtilde\theta_t - \eta_t\cdot\sum_{k = 1}^K \frac{|\calD_{\protector}^{(k)}|}{\sum_{k = 1}^K |\calD_{\protector}^{(k)}|}\wtilde g_{t}^{(k)}
    \end{align}
    After that, the server dispatches $\wtilde\theta_{t+1}$ to all clients and follows the above steps for the next round.
\end{itemize}

}

\subsection{Threat model}
\label{subsec:thread_model}
{
We consider the server is a \textit{semi-honest} attacker, he/she follows the FL protocol, yet may execute privacy attacks on exposed data to deduce the private information of other participants, e.g., gradient inversion attacks proposed by \cite{zhu2019dlg}. Specifically, the attacker has two main tasks, as shown in Figure \ref{fig:thread_model}. 
\begin{figure}[!ht]
    \centering
    \includegraphics[width=1\linewidth]{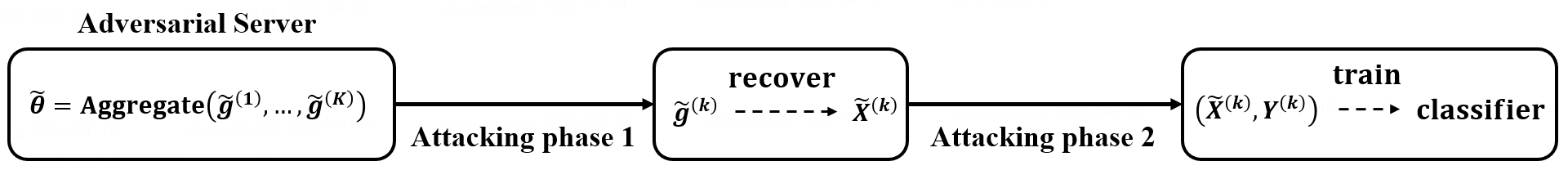}
    \caption{Threat Model.}
    \label{fig:thread_model}
\end{figure}

\begin{itemize}
        \item Attacking Phase 1: Out of curiosity, the attacker takes the distorted gradient $\widetilde{g}^{(k)}$ of client $k$, and utilizes different attack approaches to recover the dataset (denoted as $\widetilde{X}^{(k)}$). The goal is to restore the original dataset as much as possible.

        \item Attacking Phase 2: The attacker wants to know and quantify the cost of the attack approach, so they simulate an experiment for this purpose. Specifically, they utilized the recovered dataset ( $\widetilde{X}^{(k)}$), with their corresponding true label $Y^{(k)}$ to train a classifier. The goal is to deduce how many samples are needed to make the classifier satisfies $(\epsilon,\delta)$-PAC learnability.
        
    \end{itemize}
}

\subsection{Main Results}
\label{subsec:main_results}

{Under the aforementioned setting, our goal is to build a unified framework in which privacy, utility, and efficiency can be measured from the PAC learning perspective. Based on that, we further formalize the key conclusion of this paper: \textbf{How to quantify the trade-off between privacy, utility, and efficiency from the perspective of PAC learning ?} To better understand the workflow of our main results, we illustrate the roadmap in Figure \ref{fig:roadmap} for convenience. 

\begin{figure}[!ht]
   \centering
\includegraphics[width=1\linewidth]{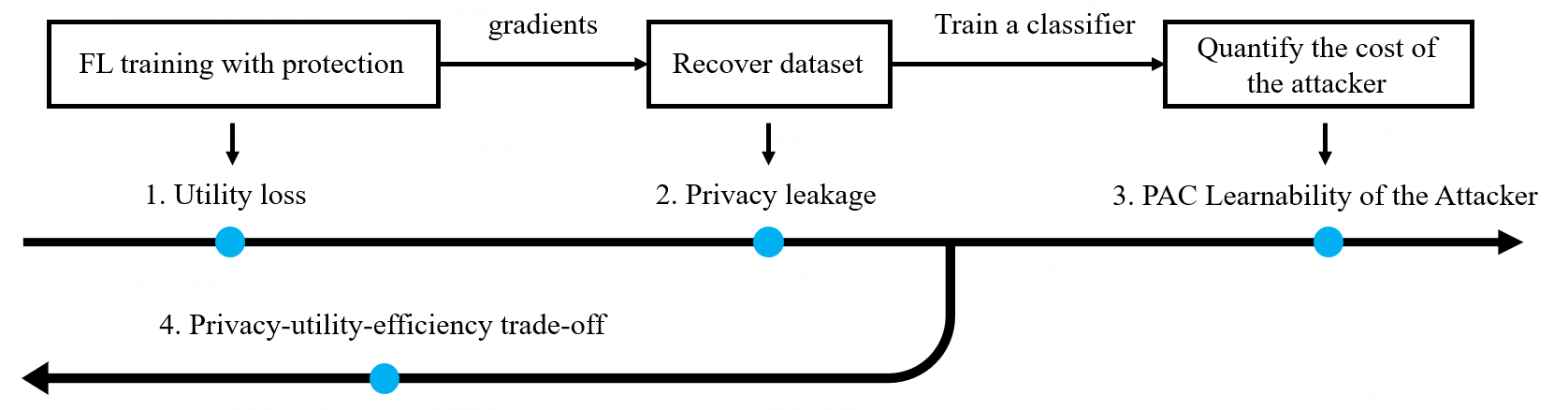}
    \caption{{Roadmap of our main results.}}
   \label{fig:roadmap}
\end{figure}

{

\begin{enumerate}
    \item The FL system follows the training procedure described in section \ref{subsec:learning_setting} to execute FL training. Because protection mechanisms are utilized to protect the information, which might potentially lead to the utility loss.

    \item Following Phase 1 described in section \ref{subsec:thread_model}, whenever the server side (attacker) receives the distorted gradients from clients, the attacker trys to recover the dataset, and inevitably leads to the privacy leakage.

    \item Following Phase 2 described in section \ref{subsec:thread_model}, the attacker further wants to evaluate the cost of the attack approach, so as to decide whether to launch an attack.

    \item Based on privacy leakage and utility loss, we expect to formalize the privacy-utility-efficiency trade-off.
\end{enumerate}
}

 Specifically, we will discuss the following questions and provide the results, as well as detailed analyses correspondingly (see Figure \ref{fig:workflow} for the illustration).
\begin{figure}[!ht]
    \centering
    \includegraphics[width=1\linewidth]{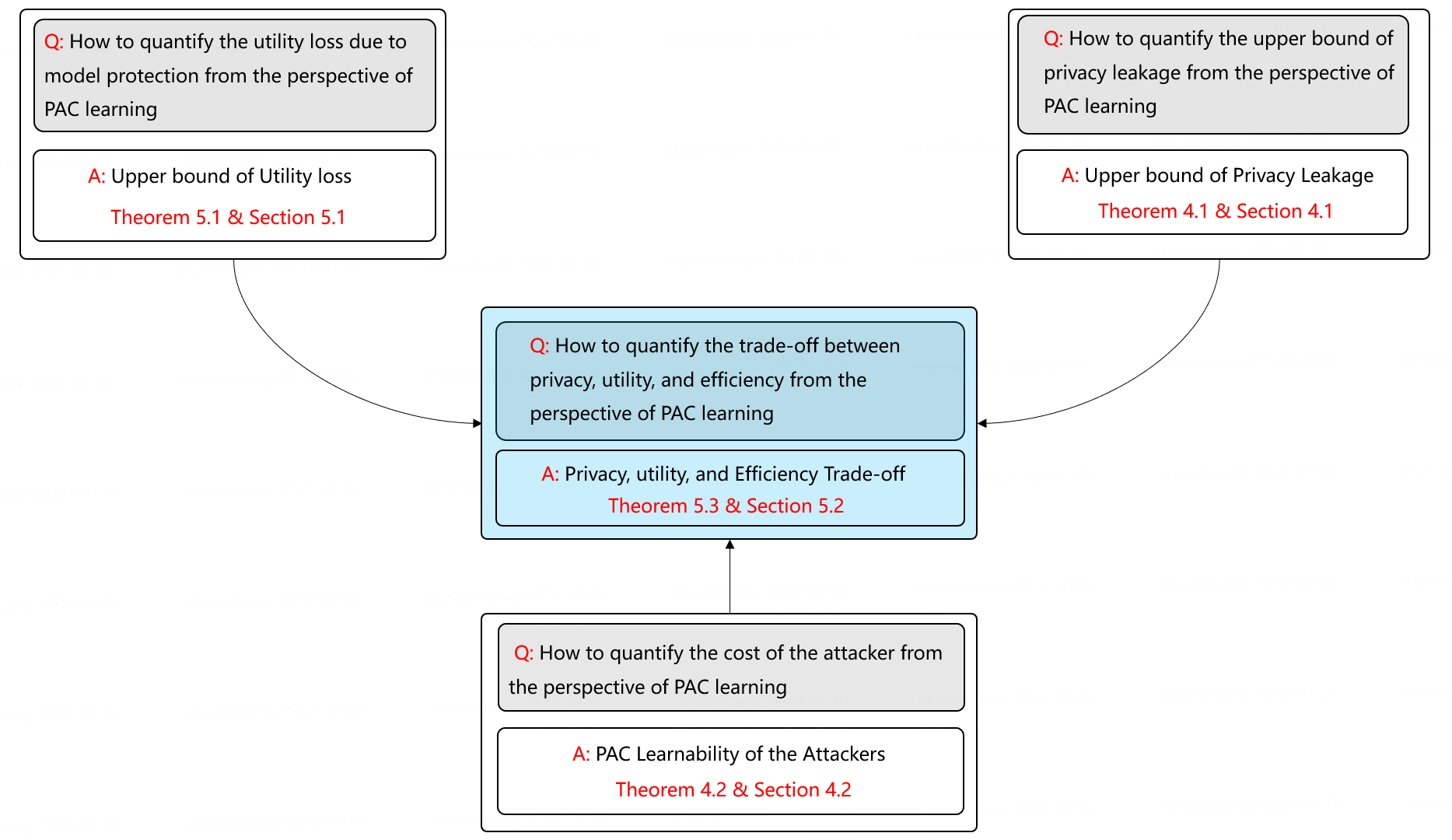}
    \caption{Illustration of our main results.}
    \label{fig:workflow}
\end{figure}

\begin{enumerate}

    \item \textbf{Q:} How to quantify the upper bound of privacy leakage (Definition \ref{defi:Privacy_Leakage}) from the perspective of PAC learning?

    \textbf{A:} Privacy
leakage measures the gap between the recovered dataset by the attacker and the original dataset. We argue that the upper bound of privacy leakage relies on the size of the training samples. Intuitively, the more training samples, the smaller the upper bound of the privacy leakage of the model, which means that, With the same probability, it is more challenging to attack the model.

        \textbf{Please refer to Theorem \ref{thm: privacy_distortion_and_datasize} for formal definition, and we provide detailed analyses in Section \ref{subsec:upper_bound_privacy_leakage}.}

    \item \textbf{Q:} How to quantify the cost of the attacker from the perspective of PAC learning?

    \textbf{A:} To quantify the cost of the attack approach, we simulate an experiment to train a classifier with the recovered dataset, we provide the lower bound of sample sizes for this purpose, which means how many samples are needed to make the classifier satisfy $(\epsilon,\delta)$-PAC learnability. 

    \textbf{Please refer to Theorem \ref{thm: exponential_lower_bound_mt} for formal definition, and we provide detailed analyses in Section \ref{sec:PAC_Learnability_of_the_Attackers}.}
    
    \item \textbf{Q:} How to quantify the utility loss (Definition \ref{defi:generalization_error_pd}) due to model protection from the perspective of PAC learning?

        \textbf{A:} Utility loss measures the gap between the performance of the model before and after adding protection. We argue that the upper bound of utility loss relies on the size of the training samples. Intuitively, the more training samples, the smaller the upper bound of the utility loss of the model, which means that the model can achieve better model performance with the same probability. 

        \textbf{Please refer to Theorem \ref{thm: relation_utility_loss} for formal definition, and we provide detailed analyses in Section \ref{sec:Generalized_error_Bound}.}

    \item \textbf{Q:} How to quantify the trade-off between privacy, utility, and efficiency from the perspective of PAC learning?

    \textbf{A:} Combining the upper bound of utility loss and the upper bound of privacy leakage, by sample complexity, we can unify these three dimensions and establish the trade-off expression between utility, privacy and efficiency. 

    \textbf{Please refer to Theorem \ref{thm: trade-off analysis for general protection mechanism} for formal definition, and we provide detailed analyses in Section \ref{subsec:Trade_off_for_General_Protection_Mechanisms}.}
\end{enumerate}
}

\section{Related Work}
\label{sec:related_work}

The concept of Federated Learning was first proposed in 2016 (\cite{mcmahan2017communication,yang2019federated}). There are many surveys and books that cover all aspects of federated learning research (\cite{yang2023federated,kairouz2021advances,li2020federated,DBLP:series/synthesis/2019YangLCKCY}). In summary, existing research can be categorized into three folds, i.e., privacy, utility, and efficiency. However, most of the research is interest in addressing each of these challenges separately, while there is relatively less focus on how to make trade-offs between these three dimensions.

In this section, we provide an overview of the trade-off problem, multi-objective optimization, and also review the basic knowledge of PAC learning, as well as how PAC learning can be applied to theoretical analysis of machine learning algorithms.

\subsection{Privacy-utility-efficiency trade-off}

The privacy-utility trade-off problem has been extensively studied in many areas, such as machine learning, database queries, etc. A well-known theory in this area is differential privacy (DP), where the goal is to prevent information leakage by adding an amount of randomness (\cite{10.1007/11787006_1}). However, the original DP definition can be overly cautious by focusing on worst-case scenarios, following this work, more flexible and generalized versions have been proposed to remedy this shortcoming, such as $(\epsilon, \delta)$-differential privacy, Rényi differential privacy (\cite{renyi_dp}), etc.

In addition to DP theory, Another research line is to cast the trade-off problem as an optimization problem. \cite{makhdoumi2013privacy} cast the privacy-utility trade-off as a convex optimization problem, where privacy leakage is modeled under the log-loss by the mutual information between private data and released data, and utility constraint is regarded as the average distortion between original and distorted data. \cite{DBLP:journals/tifs/SankarRP13} utilized rate-distortion theory to develop privacy-utility trade-off region for i.i.d. data sources with known distribution. \cite{DBLP:conf/nips/ChenKO20} develop novel encoding and decoding
mechanisms that simultaneously achieve optimal privacy and communication
efficiency in various canonical settings under distributed learning scenarios. \cite{Zhong2022PrivacyUtilityT} investigates the privacy-utility trade-off problem
for seven popular privacy measures in two different scenarios:
global and local privacy.

Under FL scenarios, many algorithms take into account the balance of privacy, utility and efficiency. For example, \cite{zhang2020batchcrypt} present BatchCrypt that substantially reduces the encryption and communication overhead caused by HE. \cite{huang2020rpn} introduce RPN to keep the model performance without loss while reducing the communication overhead. However, such researches are case by case and have no unified theoretical analysis. To fill this gap, {\cite{zhang2022trading} proposed No-Free-Lunch (NFL) theorem with a similar motivation to this paper, where privacy, utility, and efficiency are measured by distortion extent, a metric to measure the data distribution difference before and after data protection. However, in most cases, the distortion extent can only be used in the protection mechanisms, which makes it not suitable for analyzing attack algorithms. Besides, the distortion extent is determined by the data distribution, which is usually unknown in advance in real-world applications.}


\subsection{Multi-objective Optimization}

Multi-objective optimization (MOO), also known as Pareto optimization (\cite{miettinen1999nonlinear,hwang2012multiple,hassanzadeh2010multi}), is an area of mathematical optimization problems involving more than one objective function to be optimized simultaneously. MOO has been widely studied, and most MOO solutions can fall into either one of the following two classes:
\begin{itemize}
    \item Priori methods (\cite{hwang2012multiple}), which require sufficient preference information before the solution process, including lexicographic optimization (\cite{sherali1983preemptive,isermann1982linear,ogryczak2005telecommunications,cococcioni2018lexicographic}), Scalarizing (\cite{6791727,Miettinen1998NonlinearMO,WIERZBICKI1982391,Chandra1983,DBLP:journals/corr/abs-2006-04655}), goal programming (\cite{charnes1955optimal,tamiz1998goal,jones2010practical,tafakkori2022sustainable}), etc.

    \item Posteriori methods, which aim at producing all the Pareto optimal solutions, including mathematical optimization (\cite{jeter1986mathematical,martins2021engineering}), evolutionary algorithms (\cite{deb2002fast,deb2013evolutionary,jain2013evolutionary,kim2004spea2+,suman2006survey,vargas2015general,lehman2011abandoning}), deep learning-based algorithms (\cite{navon2020learning,liu2021profiling}), etc.
\end{itemize}

FL is a new application area for MOO. \cite{DBLP:journals/corr/abs-2102-08288} proposes a federated data-driven evolutionary optimization framework that is able to perform data-driven optimization when the data is distributed on multiple devices. \cite{DBLP:journals/corr/abs-2106-12086} extends this work and focuses on MOO; Some other research works utilize the MOO to deal with the trade-off among utility, efficiency, robustness, and privacy under FL setting,  \cite{liu2021multi} proposes MORAS, a multi-objective robust architecture search algorithm to balance both accuracy and robustness in the presence of multiple adversarial attacks. \cite{yin2022predictive} applies multi-objective optimization to federated split learning, the goal is to obtain the trade-off solution between the training time and energy consumption. \cite{zhu2019multi} aims to optimize the structure of the neural network models in FL using a multi-objective evolutionary algorithm to simultaneously minimize efficiency and utility. \cite{morell2022optimising} optimizes communication overhead in FL by modeling it as a multi-objective problem and applying NSGA-II (\cite{deb2002fast}) to solve it.

\subsection{Probably Approximately Correct learning}

Probably Approximately Correct (PAC) learning was first proposed in 1984 by Leslie Valiant\ (\cite{valiant1984theory}). Briefly, PAC learning is a framework for the mathematical analysis of machine learning, under which the learner selects a hypothesis function. The goal is that, with high probability, the selected function will have low generalization error.

PAC learning has been extensively studied and applied to machine learning algorithm analysis. For example, \cite{zhang2010risk} present the risk bounds for Levy processes without Gaussian components in the PAC-learning framework. \cite{balcan2012distributed} consider the problem of PAC-learning from distributed data and analyze fundamental communication complexity questions involve, as well as provide general upper and lower bounds on the amount of communication needed to learn well. \cite{10.5555/1577069.1577076} studied the properties of the agnostic learning framework, which is a natural generalization of PAC learning. \cite{561497cd439948c4ab67ad629dbe7b99} extend the PAC learning framework to account for the presence of an adversary, and seek to understand the limits of what can be learned in the presence of an evasion adversary.

There are also some studies that apply PAC learning to distributed learning scenarios.  \cite{NIPS2017_186a157b} consider a collaborative learning scenario, in which k players attempt to learn the same underlying concept, they present lower bounds and upper bounds on the sample complexity of collaborative PAC learning. However, they only consider realizable setting. Further, \cite{nguyen2018improved} extend the results to the realizable and non-realizable settings. \cite{blum2021one} analyzed how the sample complexity of federated learning may be affected by the agents’ desire to keep their individual sample complexities low in the PAC learning setting.





\section{On the Learnability of the Attacker}
\label{sec:on_the_learnability_of_the_attacker}
We formulate two phases of learning process for the semi-honest attackers (see Figure \ref{fig:thread_model}). During Learning Phase 1, the attacker infers the dataset of the protector upon observing the released model information. During Learning Phase 2, with the inferred dataset, the attacker trains a classifier. To provide a unified analysis framework for trade-off between utility, privacy and efficiency for protection mechanisms including \textit{Randomization Mechanism} and \textit{Homomorphic Encryption Mechanism}, we introduce the following two distortion extents during the training process of federated learning, called \textit{two-way distortion extent} and \textit{uplink distortion extent} (see \pref{defi: uplink_distortion_extent} and \pref{defi: two_way_distortion_extent}).

\subsection{Upper Bound for Privacy Leakage} 

\label{subsec:upper_bound_privacy_leakage}


In this section, we provide theoretical analysis for Learning Phase 1. To illustrate the Learning Phase 1 of the attacker, we elaborate on the goal of the attacker, the evaluation metric of the learning separately. Let $\wtilde g^{(k)}$ represent the distorted gradient, and $g^{(k)}$ represent the true gradient of client $k$. We denote $\calD_{\protector}^{(k)} = \{\textbf{X}^{(k)}, \textbf{Y}^{(k)}\} = \{(X_1^{(k)},Y_1^{(k)}),\cdots, (X_{m_k}^{(k)},Y_{m_k}^{(k)})\}$ as the dataset of client $k$. Assume the server is a semi-honest attacker, and is aware of the label information $\{Y_1^{(k)}, \cdots, Y_{m_k}^{(k)}\}$.

\paragraph{Learning Phase $1$ of the Attacker}
\begin{itemize}
    \item With the private dataset $\calD_{\protector}^{(k)}$ sampled from the dataset of protector $k$, the protector generates the true gradient $g^{(k)} = \frac{1}{|\calD_{\protector}^{(k)}|}\sum_{i = 1}^{|\calD_{\protector}^{(k)}|} \nabla\calL(\theta, X_{i}^{(k)}, Y_i^{(k)})$;
    \item To protect the private information, protector $k$ uploads the distorted gradient $\wtilde g^{(k)}$ to the server; 
    \item Upon observing the distorted information $\wtilde g^{(k)}$, the semi-honest attacker infers the private information.
\end{itemize}


The recovered dataset of the attacker is denoted as $\calD_{\attacker}^{(k)} = \{(\wtilde X_1^{(k)}, Y_1^{(k)}),\cdots, (\wtilde X_{m_k}^{(k)}, Y_{m_k}^{(k)})\}$, and $|\calD_{\attacker}^{(k)}| = |D_{\protector}^{(k)}|$. Let $X_{t,i}^{(k)}$ represent the $i$-th data recovered by the attacker at the $t$-th round of the optimization algorithm, and $X_{i}^{(k)}$ represent the $i$-th original data. For Phase 1, the goal of the attacker is to minimize the gap between the estimated data and the true data, i.e.,
\begin{align}
    \frac{1}{|\calD_{\attacker}^{(k)}|}\sum_{ i = 1}^{|\calD_{\attacker}^{(k)}|}\frac{1}{T}\sum_{t = 1}^T \frac{||X_{t,i}^{(k)} - X_{i}^{(k)}||}{D}.
\end{align}


For semi-honest attackers, the privacy leakage is measured using the gap between the estimated dataset and the original dataset.
\begin{defi}[Privacy Leakage]
\label{defi:Privacy_Leakage}
Let $X_{t,i}^{(k)}$ represent the $i$-th data recovered by the attacker at the $t$-th round of the optimization algorithm, and $X_{i}^{(k)}$ represent the $i$-th original data. The privacy leakage is measured using  
\begin{align}
    \epsilon_p^{(k)} = 1 - \E\left[\frac{1}{|\calD_{\attacker}^{(k)}|}\sum_{ i = 1}^{|\calD_{\attacker}^{(k)}|}\frac{1}{T}\sum_{t = 1}^T \frac{||X_{t,i}^{(k)} - X_{i}^{(k)}||}{D}\right],
\end{align}
where the expectation is taken with respect to the randomness of the mini-batch $\calD_{\attacker}^{(k)}$.\\
\textbf{Remark:} We assume that $||X_{t,i}^{(k)} - X_{i}^{(k)}||\le D$. Therefore, $\epsilon_p^{(k)}\in [0,1]$.
\end{defi}

\begin{defi}[Distorted Feature and Distorted Information]
Let $\calL$ be the loss function, and $Y$ be the feature information. We denote $\wtilde g$ as the distorted information, and ${\wtilde X}$ as the distorted feature satisfying that $\nabla \calL(\theta, \wtilde X, Y) = \wtilde g$.
\end{defi}
Now we introduce an important measurements related to privacy leakage called the \textit{\textit{uplink distortion extent}}.
   \begin{defi}[Uplink Distortion Extent]\label{defi: uplink_distortion_extent}
      The process of sending updated parameter from clients to the server is referred to as uplink. We denote $\theta^{(k)}$ as the original information of client $k$, and $\wtilde{\theta}^{(k)}$ as the distorted information uploaded from client $k$ to the server. The \textit{uplink distortion extent} of client $k$ is $\Delta_{\text{up}}^{(k)} = \| \theta^{(k)} - \wtilde{\theta}^{(k)}\|$, which measures the gap between the original information and the distorted information, it is related to the amount of privacy leakage and is referred to as the \textit{uplink distortion extent}.

      \begin{figure}[!ht]
    \centering
    \includegraphics[width=0.9\linewidth]{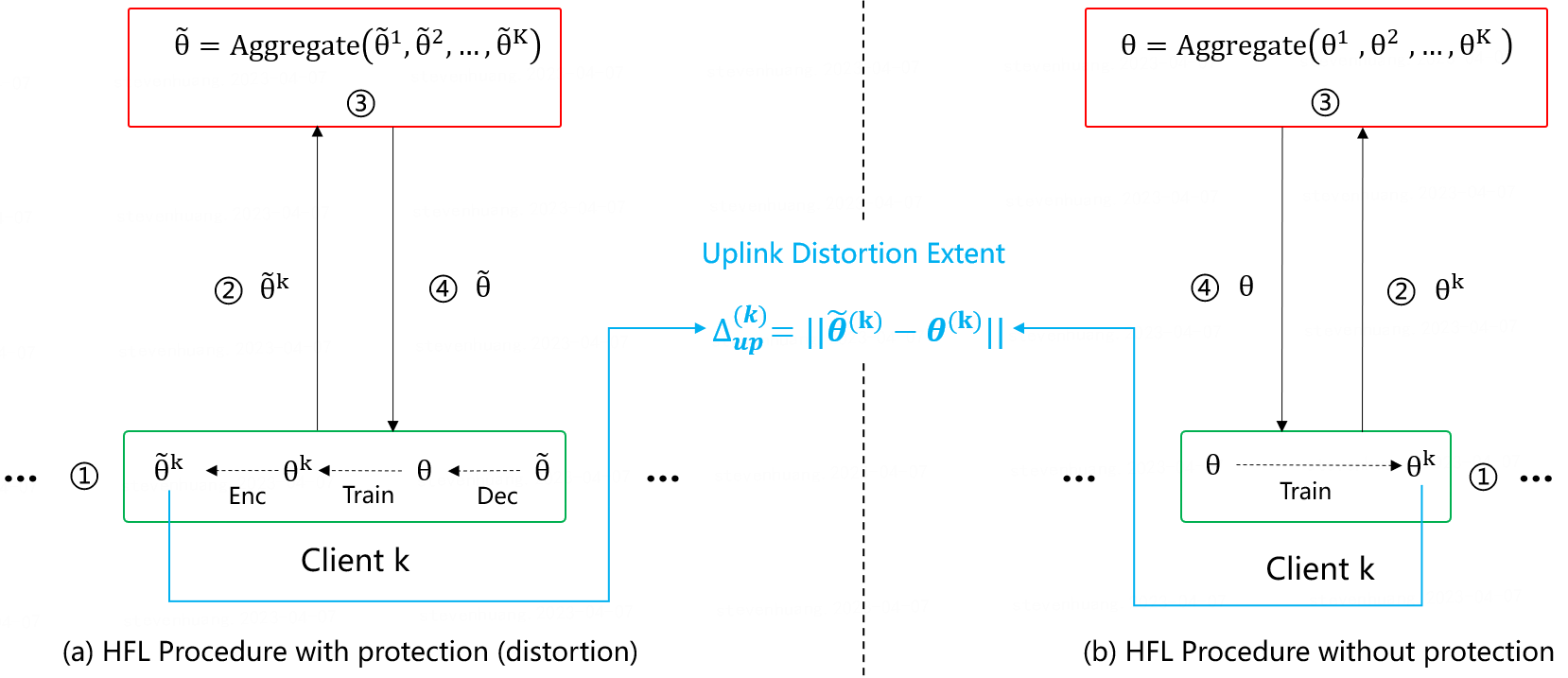}
    \caption{Illustration of \textit{uplink distortion extent}. The left side is HFL training procedure with protection, and the right side is HFL training procedure without protection.}
    \label{fig:uplink_distortion}
\end{figure}
   \end{defi}

\begin{assumption}\label{assump: two-sided Lipschitz}
   Assume that $||X_1 - X_2||\le D$ for any two data. Assume that $c_a ||\nabla \calL(\theta, X_1, Y) - \nabla \calL(\theta, X_2, Y)||\le ||X_1 - X_2||\le c_b ||\nabla \calL(\theta, X_1, Y) - \nabla \calL(\theta, X_2, Y)||$.
\end{assumption}

Let $d$ represent the dimension of the parameter, and $T$ represent the total number of rounds.  The optimization algorithms that guarantee asymptotically optimal regret are all applicable for our analysis, the examples include:
\begin{itemize}
    \item AdaGrad algorithm proposed by \cite{duchi2011adaptive}, achieves an optimal $\Theta(\sqrt{T})$ regret bound. Specifically, the regret bound is $O(\max\{\log d, d^{1-\alpha/2}\sqrt{T}\})$, where $\alpha\in (1,2)$.
    \item Adam algorithm introduced by \cite{kingma2014adam}, achieves an optimal $\Theta(\sqrt{T})$ regret bound. Specifically, the regret bound is $O(\log d\sqrt{T})$ with an improved constant.
\end{itemize}

The two motivating examples lead to the following assumption.

\begin{assumption}[Upper and Lower Bounds for Optimization Algorithms]\label{assump: bounds_for_optimization_alg}
   Let $T$ represent the total number of the learning rounds of the semi-honest attacker. Suppose that $c_0\cdot T^{1/2} \le \sum_{t = 1}^T ||\nabla \calL(\theta, X_t, Y_t) - \nabla \calL(\theta, \wtilde X_t, Y_t)|| = \Theta(T^{1/2}) \le c_2\cdot T^{1/2}$, where $X_t$ represents the dataset reconstructed by the attacker at round $t$, and ${\wtilde X}$ represents the dataset satisfying that $\nabla \calL(\theta, \wtilde X, Y) = \wtilde g$.
\end{assumption}

The reconstruction attack merits in-depth research (\cite{geng2021towards}). However, there is a basic lack of theoretical knowledge regarding how and when gradients can result in the unique recovery of original data (\cite{zhu2020r}). The following theorem illustrates a quantitative relationship between privacy leakage $\epsilon_p^{(k)}$, the size of the training set $|\calD_{\protector}^{(k)}|$ and the \textit{uplink distortion extent} $\Delta_{\text{up}}^{(k)}$. See Appendix \ref{appendix: analysis_for_privacy_distortion_and_datasize}
for more detailed analysis and proof.

\begin{thm}[Upper Bound for Privacy Leakage]\label{thm: privacy_distortion_and_datasize}
   Let $T$ represent the total number of the learning rounds of the semi-honest attacker. Let \pref{assump: two-sided Lipschitz} and \pref{assump: bounds_for_optimization_alg} hold. Assume that $\Delta_{\text{up}}^{(k)}\ge \frac{2 c_2\cdot c_b}{c_a\sqrt{T}}$, where $c_a, c_2$ and $c_b$ are introduced in \pref{assump: two-sided Lipschitz} and \pref{assump: bounds_for_optimization_alg}. With probability at least $1 - \gamma^{(k)}$, 
  \begin{align}
    \epsilon_p^{(k)}\le 1 + \sqrt{\frac{\ln(2/\gamma^{(k)})}{2|\calD_{\protector}^{(k)}|}} - \frac{c_a}{2D}\cdot\Delta_{\text{up}}^{(k)}.
  \end{align}
\end{thm}

\subsection{PAC Sample Complexity Lower Bound of the Attacker}
\label{sec:PAC_Learnability_of_the_Attackers}

 In this section, we focus on Learning Phase 2 of the attacker, and analyze PAC learnability for the semi-honest attackers. We measure privacy using the distance between the inferred data and the private data, which is referred to as data privacy. Aside from the physical distance, the utility of the inferred data is also a crucial component. To evaluate the utility, we propose Learning Phase 2 for the attacker.

We are curious about the least amount of samples required for the attacker to achieve an error of $\epsilon$ with probability at least $1 - \delta$. We provide a lower bound for the number of samples for achieving PAC learnability in terms of the amount of distortion, and show that an exponential amount of samples is necessary to achieve PAC learnability.


   

\paragraph{Learning Phase $2$ of the Attacker}


We denote the dataset of client $k$ recovered by the attacker as ${\calD_{\attacker}^{(k)}} = \{(\wtilde X_1^{(k)}, Y_1^{(k)}),\cdots, (\wtilde X_{m_k}^{(k)}, Y_{m_k}^{(k)})\}$, and $|\calD_{\attacker}^{(k)}| = |\calD_{\protector}^{(k)}|$.

Let $h\in\calH$ represent the classifier of the learner, and $c$ represent the true classifier.
\begin{itemize}
    \item The inferred dataset of the attacker is denoted as ${\calD_{\attacker}^{(k)}} = \{(\wtilde X_1^{(k)}, Y_1^{(k)}),\cdots, (\wtilde X_{m_k}^{(k)}, Y_{m_k}^{(k)})\}$;
    \item With the inferred dataset ${\calD_{\attacker}^{(k)}}$, the attacker trains a classifier $h\in\calH$;
    \item Sample a true labeled example $(\truextest, \trueytest)$. 
\end{itemize} 

The learning problem of the attacker is specified as $\calF = (\calX, \calY, \calC, \calP, \calH)$, where $\calX$ represents the instance set, $\calY$ represents the label set, $\calC$ represents the concept class, and $\calH$ represents the hypothesis class. Let $\calP$ represent a class of distributions over instances $\calX$. Let the probability space be $(\calX, \distribution, \distance)$, where $\distribution\in\calP$ represents a distribution, and $\distance$ represents a metric for measuring the distance. Let $c\in\calC$ represent the true classifier. The goal of the attacker is to train a classifier based on the inferred dataset, and minimize the risk of the classifier on the true dataset. The question we are concerned about is: \textit{when is it possible for the attacker to achieve the $(\epsilon, \delta)$-PAC learning?} $(\epsilon, \delta)$-PAC learnability requires the error of the classifier $h$ is at most $\epsilon$ on the original dataset under bounded distortion of the model parameter, with probability at least $1-\delta$.


\paragraph{Risk of the Attacker}   
    The risk of the attacker is taken with respect to the distribution $\wtilde P$ and evaluated on the inferred dataset $\wtilde \calD$,
       \begin{align}
        \risk(h, c) = \Pr_{\distpara\sim\distdistribution}[h(\distpara)\neq c(\distpara)]. 
    \end{align}
    


$\risk$ depicts the classification loss with the true information, and we also need a quantity that measures the classification loss with the distorted information. The attacker could perturb the victim's training data but does not know the model initialization, training routine, or architecture of the victim. The victim trains a new model with the poisoned dataset. The performance of the attacker is evaluated according to the accuracy of the victim model on the original (uncorrupted) dataset (\cite{diochnos2019lower, fowl2021adversarial}). 
\paragraph{Adversarial Risk of the Attacker}
   The adversarial risk of the attacker is defined as
\begin{align}
    \advrisk(h,c,\Delta) = \Pr_{\distpara\sim\distdistribution}[\exists \truepara \in\calX: d(\truepara, \distpara)\le \Delta, h(\truepara)\neq c(\truepara)],
\end{align}
where the budget $\Delta$ measures the amount of distortion that the protector could introduce considering the constraint on the utility loss. The limitation is measured using some metric $d$ defined over the input space $\calX$ and a budget $\Delta$ on the amount of perturbations that the protector can introduce.

\paragraph{The Goals of the Attacker and the Protector}
The goal of the attacker is to learn a robust $h$ to minimize the adversarial risk from the inferred dataset. The goal of the protector is to increase the adversarial risk of the attacker, with a small amount of distortion on the transmitted model parameter. The protector aims at finding the label of the original untampered point $\truepara$ by only having its corrupted version $\distpara$. Therefore, the success criterion of the protector is to reach $$d(\truepara, \distpara)\le \Delta, h(\truepara)\neq c(\truepara).$$

\begin{assumption}\label{assump: about_close_functions}
For all $1\ge\epsilon\ge 2^{-\Theta(d)}$, there exist two classifiers $c_1, c_2$ (from $\calX$ to $\calY$) satisfying that $\Pr_{\distpara\sim P}[c_1(\distpara)\neq c_2(\distpara)] = \Theta(\epsilon)$, where $P$ represents a distribution over $\calX$.
\end{assumption}

\begin{assumption}\label{assump: normal_levy_faimily}
Assume that the domain of the protected information is a Normal Levy Family. 
\end{assumption}








The classifier is trained using the original dataset and evaluated on the distorted dataset. The following theorem provides a lower bound for sample complexity for PAC learnability of an attacker. It measures the relationship between the number of samples and the distortion extent, and also provides a relationship between distortion extent and model privacy. The high-level idea is: If the size of the inferred dataset $\wtilde \calD$ is not exponentially large, then the risk of any learner on the inferred dataset is significantly larger than $0$, and then the risk of any learner on the original dataset is at least $2^{\Omega(\Delta)}$.






\begin{thm}[Lower bound for PAC Learnability of the Attackers]\label{thm: exponential_lower_bound_mt}
    Let \pref{assump: two-sided Lipschitz}, \pref{assump: about_close_functions} and \pref{assump: normal_levy_faimily} hold. Assume $1 > \delta > 0.5$. For any $(\epsilon, \delta)$-PAC learning attacker, the number of samples associated with protector $k$ should be at least
   \begin{align}
       |\calD_{\attacker}^{(k)}|\ge \min\{2\delta - 1, 1-\epsilon\}\cdot 2^{c_a^2{\Delta_{\text{up}}^{(k)}}^2},
   \end{align}
   where $\Delta_{\text{up}}^{(k)}$ represents the \textit{uplink distortion extent} of client $k$, and $c_a$ is introduced in \pref{assump: two-sided Lipschitz}.
\end{thm}

\textbf{Remark:} 
Distinct corruption settings and risk definitions correspond to distinct levels of lower bounds. For example, linear, polynomial and exponential with respect to the distortion extent.




\section{On the Utility Loss of the Protector in Privacy-Preserving Federated Learning}
\label{sec:utility_loss_analysis}

 

\cite{zhang2022no, zhang2022trading} measured utility via the model performance on the given dataset of the client. In contrast, we propose a more general measurement for the utility of the client, assuming the dataset of each client is sampled i.i.d. from an unknown distribution. Before elaborating on the measurement for utility loss, we introduce the \textit{generalized protected loss} and the \textit{empirical original loss} first. For parameter distortion, the \textit{generalized protected loss} is defined as
\begin{align}
  \calL^{(k)}_{\text{exp}}(\theta + \delta) = \E_{Z\sim P}[\calL^{(k)}(\theta + \delta , Z)],  
\end{align}
where $\delta$ is the distortion applied to the model parameter $\theta$ for protecting privacy.  $\delta$ can take various forms, including noise and compression, and it is typically bounded by private budget, communication cost, and other constraints. The bounds of $\delta$ are generally determined according to the specific domain. We denote the \textit{empirical original loss} as 
 \begin{align}
     \calL^{(k)}_{\text{emp}}(\theta, Z) = \frac{1}{m_k}\sum_{i = 1}^{m_k} \calL^{(k)}(\theta, Z_i^{(k)}).
 \end{align}
 

With the aforementioned loss functions, we are now ready to introduce the definition of utility loss. 
\begin{defi}[Utility Loss]
\label{defi:generalization_error_pd}
   Let $\calD_{\protector}^{(k)} = \{Z_1^{(k)}, \cdots, Z_{m_k}^{(k)}\}$. We denote the \textit{empirical original loss} as $\calL^{(k)}_{\text{emp}}(\theta, \calD_{\protector}^{(k)}) = \frac{1}{m_k}\sum_{i = 1}^{m_k} \calL^{(k)}(\theta, Z_i^{(k)})$, and denote the \textit{generalized protected loss} as $\calL^{(k)}_{\text{exp}}(\theta + \delta) = \E_{Z\sim P}[\calL^{(k)}(\theta + \delta , Z)]$. The utility loss of client $k$ is measured using 
   \begin{align}
       \epsilon_u^{(k)} &= |\calL^{(k)}_{\text{exp}}(\theta + \delta) - \calL^{(k)}_{\text{emp}}(\theta)|.
   \end{align}
\end{defi}


Another important metric we consider in the PAC framework is the training efficiency, which is defined as follows. 
\begin{defi}[Training Efficiency]
\label{defi:Training_Efficiency}
   The training efficiency of client $k$ is measured using the size of the training set of client $k$, denoted as $\epsilon_e^{(k)}$. 
\end{defi}

\subsection{On the Price of Preserving Privacy}\label{sec:Generalized_error_Bound}


In this section, we provide an upper bound for the utility loss of the protector with i.i.d. samples. The utility loss measures the gap between the performance of the model on a shifted distribution and that on the training dataset.
We focus on the generalization bound of learning algorithms by investigating the utility loss on the distorted information from the training set. First we introduce an important measurement related to the utility loss called \textit{two-way distortion extent}.
\begin{defi}[Two-way Distortion Extent]\label{defi: two_way_distortion_extent}
      The process of sending updated parameter from clients to the server is referred to as uplink. The process of sending parameters from the server to clients is referred to as downlink. We denote $\theta$ as the original information, and $\wtilde{\theta}$ as the distorted information downloaded from the server and initialized for local optimization. The \textit{two-way distortion extent} 
      \begin{align}
         \Delta_{\text{two}} = \| \text{Dec}(\wtilde\theta) - \theta\|
      \end{align}
      measures the gap between the original information and the distorted information during the whole process of uplink and downlink, and is related to utility loss.
\end{defi}

\begin{figure}[!ht]
    \centering
    \includegraphics[width=1\linewidth]{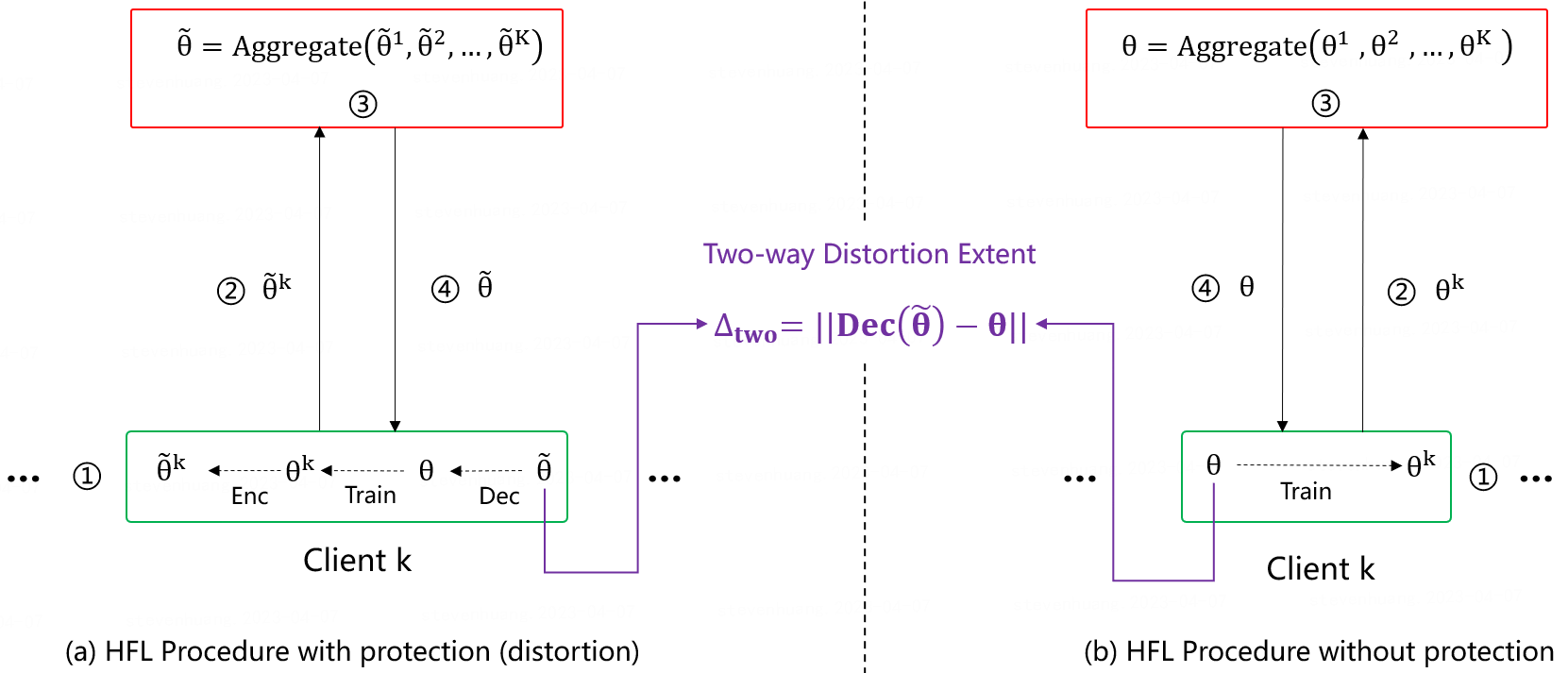}
    \caption{Illustration of \textit{two-way distortion extent}. The left side is HFL training procedure with protection, and the right side is HFL training procedure without protection.}
    \label{fig:two_way_distortion}
\end{figure}

\begin{defi}[$\lambda^{(k)}$-cover]\label{defi: r_cover}
A $\lambda^{(k)}$-cover of $(\calX_{\protector}^{(k)}, \|\cdot\|_2)$ is a point set $\{\mathbf u\}\subset\calX_{\protector}^{(k)}$ such that for any $\mathbf v\in\calD_{\protector}^{(k)}$, there exists $\mathbf u$ satisfies $\|\mathbf u - \mathbf v\|_2\le \lambda^{(k)}$.
\end{defi}


Now we analyze the generalization bounds against semi-honest attackers, with \textit{Randomization Mechanism} be the protection mechanism. The following theorem measures the price of preserving privacy in terms of utility loss using the \textit{two-way distortion extent}, assuming the training set of client $k$ consists of $|\calD_{\protector}^{(k)}|$ independent and identically distributed samples that follows distribution $P$.



\begin{thm}[Utility Loss against Semi-honest Attackers]\label{thm: relation_utility_loss}
    Let $\theta\in\R^d$ represent the model parameter. Assume that $\max_{Z\in\calD_{\protector}^{(k)}}|\calL^{(k)}(\theta + \delta , Z)|\le M^{(k)}$. We denote the empirical loss as $\calL^{(k)}_{\text{emp}}(\theta, Z) = \frac{1}{m_k}\sum_{i = 1}^{m_k} \calL^{(k)}(\theta, Z_i^{(k)})$, and denote the expected loss as $\calL^{(k)}_{\text{exp}}(\theta + \delta) = \E_{Z\sim P}[\calL^{(k)}(\theta + \delta , Z)]$. With probability at least $1 - \eta^{(k)}$,
   \begin{align}
       \epsilon_u^{(k)} &= |\calL^{(k)}_{\text{exp}}(\theta + \delta) - \calL^{(k)}_{\text{emp}}(\theta)|\nonumber\\
       &\le C\cdot\lambda^{(k)} + C\cdot\Delta_{\text{two}} + M^{(k)}\sqrt{\frac{2N^{(k)}\ln 2 + 2\ln(1/\eta^{(k)})}{|\calD_{\protector}^{(k)}|}},
   \end{align}
   where $\calD_{\protector}^{(k)} = \{Z_1^{(k)}, \cdots, Z_{m_k}^{(k)}\}$ represents the training set of client $k$, $\lambda^{(k)} > \Delta_{\text{two}}$ represents a constant, $M^{(k)}$ is a constant, $N^{(k)} = (2d)^{2D/(\lambda^{(k)})^2 + 1}$ represents the covering number of $\lambda^{(k)}$-cover (Definition \ref{defi: r_cover}), $\Delta_{\text{two}}$ represents the \textit{two-way distortion extent} (Definition \ref{defi: two_way_distortion_extent}), and $\|\delta\| = \Delta_{\text{two}}$.
\end{thm}


\subsection{Trade-off Between Utility, Privacy and Efficiency for General Protection Mechanisms}
\label{subsec:Trade_off_for_General_Protection_Mechanisms}

Before elaborating on the theoretical trade-off for the protection mechanism, we first derive the relationship between the \textit{two-way distortion extent} and the \textit{uplink distortion extent}. The following theorem illustrates the relationship between $\delta$ and $\Delta_{\text{two}}$, shows that $\Delta_{\text{up}}^{(k)}$ is equal to the norm of $\delta^{(k)}$, and further highlights the relationship between the \textit{uplink distortion extent} and the \textit{two-way distortion extent}. 


\begin{thm}\label{thm: relation_between_two_link_distortion_and_delta_mt}
   Let $\Delta_{\text{two}}$ denote the \textit{two-way distortion extent} (\pref{defi: two_way_distortion_extent}), and $\Delta_{\text{up}}^{(k)}$ denote the \textit{uplink distortion extent} (\pref{defi: uplink_distortion_extent}). The protector $k$ distorts the gradient according to \pref{eq: distortion_approach}, $\delta^{(k)}$ is introduced therein, and $\delta = \wtilde\theta - \theta$. Assume that $\text{Dec}(\wtilde\theta) = \wtilde\theta$, then we have
   \begin{align}
          \Delta_{\text{two}} = \|\delta\|,\quad \text{and} \quad\Delta_{\text{up}}^{(k)} = \|\delta^{(k)}\|.
      \end{align}
Furthermore, assume that $\delta^{(1)} = \cdots = \delta^{(K)}$. Then we have
\begin{align}
    \Delta_{\text{two}} = \Delta_{\text{up}}^{(k)}.
\end{align}
\end{thm}


Combining \pref{thm: privacy_distortion_and_datasize}, \pref{thm: relation_utility_loss} and \pref{thm: relation_between_two_link_distortion_and_delta_mt}, our main result is illustrated in the following theorem, which establishes a trade-off between utility, privacy and efficiency that acts as a paradigm for the design of protection mechanism against learning-based and semi-honest attacks. Please refer to \pref{sec: trade-off analysis for general protection mechanism} for the full proof. 


\begin{thm}[Utility, Privacy and Efficiency Trade-off for General Protection Mechanisms]\label{thm: trade-off analysis for general protection mechanism}
   Let $M^{(k)}$ be a constant satisfying that $\max_{Z\in\calD_{\protector}^{(k)}}|\calL(\theta + \delta, Z)|\le M^{(k)}$. Assume that $\Delta_{\text{two}} = \Theta(\Delta_{\text{up}}^{(k)})$, then there exists a constant $L$, satisfying that $\Delta_{\text{two}}\le L\cdot\Delta_{\text{up}}^{(k)}$. With probability at least $1 - \eta^{(k)} - \sum_{k =1}^K \gamma^{(k)}$, we have that
   \begin{align}
      \epsilon_u^{(k)} 
       &\le C\cdot L\cdot\frac{1}{K}\sum_{k = 1}^K \left(\frac{2D}{c_a}\cdot (1 - \epsilon_p^{(k)}) + \frac{2D}{c_a}\cdot\sqrt{\frac{\ln(1/\gamma^{(k)})}{\epsilon_e^{(k)}}}\right)\nonumber\\
       & + M^{(k)}\sqrt{\frac{2 N^{(k)}\ln 2 + 2\ln(1/\eta^{(k)})}{\epsilon_e^{(k)}}}.
   \end{align}
\end{thm}

\subsection{Trade-off Analysis for Randomization Mechanism and Homomorphic Encryption Mechanism}
\label{subsec:trade_off_random_he}
In this section, we quantify the trade-off between utility, privacy and efficiency for two commonly-used privacy-preserving mechanisms: \textit{Randomization Mechanism} (\cite{geyer2017differentially,truex2020ldp,abadi2016deep}) and \textit{Homomorphic Encryption Mechanism} (\cite{gentry2009fully,batchCryp}).



\subsubsection{Utility, Privacy and Efficiency Trade-off for Randomization Mechanism}

Now we apply our theoretical framework to analyze the trade-off for \textit{Randomization Mechanism}. Please refer to Appendix \pref{appendix: analysis_for_utility_privacy_efficiency} for the full proof.

\begin{thm}[Utility, Privacy and Efficiency Trade-off for \textit{Randomization Mechanism}]\label{thm: utility_privacy_efficiency}
   Assume that $\max_{Z\in\calD_{\protector}^{(k)}}|\calL(\theta + \delta, Z)|\le M^{(k)}$. For \textit{Randomization Mechanism}, with probability at least $1 - \eta^{(k)} - \sum_{k =1}^K \gamma^{(k)}$, we have that
   \begin{align}\label{eq: utility_privacy_efficiency}
       \epsilon_u^{(k)} 
       &\le (2 + \rho) C\cdot\frac{1}{K}\sum_{k = 1}^K \frac{2D}{c_a}\cdot (1 - \epsilon_p^{(k)}) + (2 + \rho) C\cdot\frac{1}{K}\sum_{k = 1}^K \frac{2D}{c_a}\cdot\sqrt{\frac{\ln(1/\gamma^{(k)})}{\epsilon_e^{(k)}}}\nonumber\\
       & + M^{(k)}\sqrt{\frac{2 N^{(k)}\ln 2 + 2\ln(1/\eta^{(k)})}{\epsilon_e^{(k)}}},
   \end{align}
   where $\rho > 0$ represents a constant.
\end{thm}



\subsubsection{Utility-Efficiency Trade-off for Homomorphic Encryption Mechanism}
Now we introduce the generalization bounds for \textit{Homomorphic Encryption} (\cite{gentry2009fully,batchCryp}). For \textit{Homomorphic Encryption Mechanism}, the \textit{two-way distortion extent} $\Delta_{\text{two}} = \| \text{Dec}(\wtilde\theta) - \theta\|= 0$. The following lemma illustrates the value of $\Delta_{\text{two}}$ for \textit{Homomorphic Encryption Mechanism}. Please refer to Appendix \pref{appendix: two_way_distortion_HE_app} for the full proof.

\begin{lem}\label{lem: two_way_distortion_HE_mt}
   Let $\Delta_{\text{two}}$ denote the \textit{two-way distortion extent} (\pref{defi: two_way_distortion_extent}). For \textit{Homomorphic Encryption Mechanism}, we have that
   \begin{align}
          \Delta_{\text{two}} = 0. 
      \end{align}
\end{lem}

Note that $\Delta_{\text{up}}^{(k)} > 0$ for \textit{Homomorphic Encryption Mechanism}. Therefore, the assumption that $\Delta_{\text{two}} = \Theta(\Delta_{\text{up}}^{(k)})$ in \pref{thm: trade-off analysis for general protection mechanism} does not hold. Therefore, we do not provide trade-off analysis for the three factors via applying \pref{thm: trade-off analysis for general protection mechanism} for \textit{Homomorphic Encryption Mechanism}. Instead, we provide a theoretical trade-off analysis for utility and efficiency, as is illustrated in the following theorem.  

\begin{thm}[Utility-Efficiency Trade-off for Homomorphic Encryption Mechanism]\label{thm: relation_utility_loss_and_he}
    Let $\theta\in\R^d$ represent the model parameter. Assume that $\max_{Z\in\calD_{\protector}^{(k)}}|\calL^{(k)}(\theta + \delta , Z)|\le M^{(k)}$. We denote the empirical loss as $\calL^{(k)}_{\text{emp}}(\theta, Z) = \frac{1}{m_k}\sum_{i = 1}^{m_k} \calL^{(k)}(\theta, Z_i^{(k)})$, and denote the expected loss as $\calL^{(k)}_{\text{exp}}(\theta + \delta) = \E_{Z\sim P}[\calL^{(k)}(\theta + \delta , Z)]$. With probability at least $1 - \eta^{(k)}$,
   \begin{align}
       \epsilon_u^{(k)} &= |\calL^{(k)}_{\text{exp}}(\theta + \delta) - \calL^{(k)}_{\text{emp}}(\theta)|\nonumber\\
       &\le C\cdot\lambda^{(k)} + M^{(k)}\sqrt{\frac{2N^{(k)}\ln 2 + 2\ln(1/\eta^{(k)})}{|\calD_{\protector}^{(k)}|}},
   \end{align}
   where $\calD_{\protector}^{(k)} = \{Z_1^{(k)}, \cdots, Z_{m_k}^{(k)}\}$ represents the training set of client $k$, $\lambda^{(k)} > \Delta_{\text{two}}$ represents a constant, $M^{(k)}$ is a constant, $N^{(k)} = (2d)^{2D/(\lambda^{(k)})^2 + 1}$ represents the covering number of $\lambda^{(k)}$-cover (Definition \ref{defi: r_cover}), $\Delta_{\text{two}}$ represents the \textit{two-way distortion extent} (Definition \ref{defi: two_way_distortion_extent}), and $\|\delta\| = \Delta_{\text{two}} = 0$.
\end{thm}


\section{Discussions}
\label{sec:discussions}
The challenge of increasing the effectiveness and efficiency of federated learning has prompted the development of a number of strategies in this area, which are classified into the following categories:
\begin{itemize}
    \item Improving Model Utility [C1]: Personalization (\cite{hard2018federated, wang2019federated}), Fine-tuning and Domain Adaptation (\cite{mansour2009domain, ben2010theory, mansour2021theory, cortes2014domain});
    \item Improving Privacy Guarantee [C2]: \textit{Randomization} (\cite{geyer2017differentially,truex2020ldp,abadi2016deep}), \textit{Sparsity} (\cite{shokri2015privacy, gupta2018distributed, thapa2020splitfed}) and \textit{Homomorphic Encryption (HE)} (\cite{gentry2009fully,batchCryp});
    \item Improving Training Efficiency [C3]: Accelerated Mini-batch SGD (\cite{cotter2011better, dekel2012optimal, lan2012optimal, stich2018local}), SCAFFOLD (\cite{karimireddy2020scaffold}).
\end{itemize}

We propose a unifying framework for analyzing the trade-off between privacy, utility and efficiency for various protection mechanisms including \textit{Randomization} (\cite{geyer2017differentially,truex2020ldp,abadi2016deep}) and \textit{Homomorphic Encryption (HE)} (\cite{gentry2009fully,batchCryp}). Within this framework, we provide a unified metric so that different protection and attacking mechanisms are comparable via sample complexity. We derive theoretical bounds for a myriad of federated learning approaches against semi-honest attackers. Specifically, \pref{thm: trade-off analysis for general protection mechanism} and \pref{thm: exponential_lower_bound_mt} shed light on algorithmic designs of protection mechanisms against semi-honest attackers with Learning Phase 1 and 2 separately. \pref{thm: trade-off analysis for general protection mechanism} shows that the preservation of privacy [C2] (efficiency [C3]) may be fundamentally at odds with the goal of improving utility [C1], and provides an explanation for the decline in model utility [C1] when privacy [C2] (efficiency [C3])-preserving strategies are implemented in real-world settings. \pref{thm: exponential_lower_bound_mt} could serve as a basis for the algorithmic proposal of a novel protection mechanism, which assures PAC-learnability for any semi-honest attackers with Learning Phase 2 could not be achieved.

\paragraph{Implications on Designing a Protection Mechanism Against Phase-1 Attacker}

We use \textit{Randomization Mechanism} as an example, and focus on the theoretical result illustrated in \pref{thm: utility_privacy_efficiency} (an application of \pref{thm: trade-off analysis for general protection mechanism} on \textit{Randomization Mechanism}). \pref{thm: utility_privacy_efficiency} provides an avenue to evaluate the number of samples required for the protector to achieve a utility loss of at most $\alpha$ with high probability, given the requirement on privacy leakage. Please refer to Appendix \ref{appendix: analysis_for_corollary_utility_privacy_efficiency} for the full proof. 
\begin{corollary}[Private PAC Learning]\label{corr: utility_privacy_efficiency}
Let $\calX$ represent the feature set, and $\calY$ represent the label set. Let $\epsilon_p$ represent the amount of privacy leaked to the learning-based and semi-honest attackers.  
    For any $\alpha > c_2 (1 - \epsilon_p)$, given $n = \Theta\left(\frac{\ln(\frac{1}{\eta})}{(\alpha - c_2 (1 - \epsilon_p))^2}\right)$ i.i.d. samples from any distribution $P$ on $\calX\times\calY$, with probability at least $1 - \eta$, the utility loss of the protector is at most $\alpha$, as is shown in the following equation:
    \begin{equation}\label{eq:private_pac_learning}
        {\rm Pr}(\epsilon_u \leq \alpha) \geq 1-\eta.
    \end{equation}
\end{corollary}

\paragraph{Implications on Designing a Protection Mechanism Against Phase-2 Attacker}

\pref{thm: exponential_lower_bound_mt} implies that the protector could guarantee not PAC learnable for any semi-honest attacker with Learning Phase 2 by adjusting the distortion extent and the size of the training set, as is illustrated in the following corollary.
\begin{corollary}[Not PAC Learnable]\label{cor: lower_bound_attacker}
   If the \textit{uplink distortion extent} of protector $k$ satisfies that $\Delta_{\text{up}}^{(k)} > \ln \left(\frac{|\calD_{\protector}^{(k)}|}{1-\epsilon}\right)$, then the classification problem associated with protector $k$ is not $(\epsilon, \delta)$-PAC learnable for any semi-honest attacker.  
\end{corollary}

It is worth noting that $|\calD_{\protector}^{(k)}|$ and $\Delta_{\text{up}}^{(k)}$ are both determined by the protector, which facilitates the proposal of a novel protection mechanism, and at the same time guarantees that any learning-based and semi-honest attackers could not achieve PAC-learnability.


\section{Conclusion and Future Work}
In this work, we model the performance of the protector and the attacker from the perspective of PAC learning, and evaluate the cost of both the protector and the attacker via sample complexity in a unified manner. Within the proposed unified framework, we measure multiple objectives including privacy, utility and efficiency along the unified evaluation metric. For the learning-based and semi-honest attacker, we analyze how many samples are needed to achieve PAC learnability (see exponential lower bound). For the defender, we analyze how many samples are needed to achieve required utility loss with high probability in the framework of PAC learning (see upper bound for utility loss). The unifying framework along with the unified measurements can further inspire the design of smart privacy-preserving federated learning algorithms that deal with the trade-offs in the solution space. We also provide an upper bound for the privacy leakage of the semi-honest attacker. The bounds for utility loss and privacy leakage in the PAC framework further establishes a trade-off between utility, privacy and efficiency. Our analysis is applicable to distinct protection mechanisms including \textit{Randomization Mechanism} and \textit{Homomorphic Encryption Mechanism} against learning-based and semi-honest attackers. Our theoretical analysis serves as a basis for the algorithmic proposal of a novel protection mechanism. It is an intriguing direction to see if our techniques could be used to quantify theoretical trade-offs against malicious attackers.


\bibliography{references}

\clearpage

\appendix
\section{Notations}
The notations used throughout this article is shown in the following table.

\begin{table*}[!ht]
\footnotesize
  \centering
  \setlength{\belowcaptionskip}{15pt}
  \caption{Table of Notation}
  \label{table: notation}
    \begin{tabular}{cc}
    \toprule
    Notation & Meaning\cr
    \midrule\
    $\epsilon_p^{(k)}$ & privacy leakage of client $k$\cr
    $\epsilon_u^{(k)}$ & utility loss of client $k$\cr
    $\calL$ & loss function\cr
    $g$ & the original gradient\cr
    $\wtilde g$ & the distorted gradient\cr
    $\Delta$ & distortion extent\cr
    $\theta$ & the original model parameter\cr
    $\wtilde \theta$ & the distorted model parameter\cr
    $\calX$ & a set of instances\cr
    $\calY$ & a set of labels\cr
    $P$ & distribution over instances $\calX$\cr
    \bottomrule
    \end{tabular}
\end{table*}

\section{Analysis for \pref{thm: privacy_distortion_and_datasize}}
\label{appendix: analysis_for_privacy_distortion_and_datasize}

\begin{lem}[Chernoff-Hoeffding Inequality]
\label{lem:hoeffdingbound}
Let $X_1, X_2, \dots, X_t$ be \textit{i.i.d.}\ random variables supported on $[0, 1]$. For any positive number $\epsilon$, we have
    \begin{align} 
    \Pr\left( \left|\frac{1}{t} \sum_{i = 1}^t X_i - 
    E\left[\frac{1}{t} \sum_{i = 1}^t X_i\right]\right| \ge \epsilon \right) \le 2 \exp \left( -2t\epsilon^2 \right).
    \end{align}
\end{lem}

\begin{thm}[Upper Bound for Privacy Leakage]\label{thm: privacy_distortion_and_datasize_app}
    Let \pref{assump: two-sided Lipschitz} and \pref{assump: bounds_for_optimization_alg} hold. Assume that $\Delta_{\text{up}}^{(k)}\ge \frac{2 c_2\cdot c_b}{c_a\sqrt{T}}$, where $T$ represents the total number of the learning rounds of the semi-honest attacker, $c_a, c_2$ and $c_b$ are introduced in \pref{assump: two-sided Lipschitz} and \pref{assump: bounds_for_optimization_alg}. With probability at least $1 - \gamma^{(k)}$, we have 
  \begin{align}
    \epsilon_p^{(k)}\le 1 + \sqrt{\frac{\ln(2/\gamma^{(k)})}{2|\calD_{\protector}^{(k)}|}} - \frac{c_a}{2D}\cdot\Delta_{\text{up}}^{(k)}.
  \end{align}
\end{thm}

\begin{proof}
   Let $X_{t,i}^{(k)}$ represent the $i$-th data of the protector $k$ that is recovered by the attacker at the $t$-th round of the optimization algorithm, and $X_{i}^{(k)}$ represent the $i$-th original data of the protector $k$. Let $\calD_{\attacker}^{(k)} = \{(X_{1}^{(k)}, Y_1^{(k)}), \cdots, (X_{m_k}^{(k)}, Y_{m_k}^{(k)})\}$ represent the mini-batch that generates $\wtilde g$. We have that
\begin{align}
    &\frac{1}{|\calD_{\attacker}^{(k)}|}\sum_{ i = 1}^{|\calD_{\attacker}^{(k)}|}\frac{1}{T}\sum_{t = 1}^T ||X_{t,i}^{(k)} - X_{i}^{(k)}||\nonumber\\
    &\ge \frac{1}{|\calD_{\attacker}^{(k)}|}\sum_{ i = 1}^{|\calD_{\attacker}^{(k)}|}\frac{1}{T}\sum_{t = 1}^T||\wtilde X_{i}^{(k)} - X_{i}^{(k)}|| - \frac{1}{|\calD_{\attacker}^{(k)}|}\sum_{ i = 1}^{|\calD_{\attacker}^{(k)}|}\frac{1}{T}\sum_{t = 1}^T||X_{t,i}^{(k)} - \wtilde X_{i}^{(k)}||\nonumber\\
    &\ge  c_a\cdot\frac{1}{T}\sum_{t = 1}^T\frac{1}{|\calD_{\attacker}^{(k)}|}\sum_{ i = 1}^{|\calD_{\attacker}^{(k)}|}||\nabla \calL(\theta, \wtilde X_{i}^{(k)}, Y_i^{(k)}) - \nabla \calL(\theta, X_{i}^{(k)}, Y_i^{(k)})||\nonumber\\ 
    &\quad - c_b\cdot\frac{1}{T}\sum_{t = 1}^T\frac{1}{|\calD_{\attacker}^{(k)}|}\sum_{ i = 1}^{|\calD_{\attacker}^{(k)}|}||\nabla \calL(\theta, X_{t,i}^{(k)}, Y_i^{(k)}) - \nabla \calL(\theta, \wtilde X_{i}^{(k)}, Y_i^{(k)})||\nonumber\\
    &\ge c_a\cdot\frac{1}{T}\sum_{t = 1}^T||\frac{1}{|\calD_{\attacker}^{(k)}|}\sum_{ i = 1}^{|\calD_{\attacker}^{(k)}|} (\nabla \calL(\theta, \wtilde X_{i}^{(k)}, Y_i^{(k)}) - \nabla \calL(\theta, X_{i}^{(k)}, Y_i^{(k)}))||\nonumber\\
    &\quad - c_b\cdot\frac{1}{T}\sum_{t = 1}^T\frac{1}{|\calD_{\attacker}^{(k)}|}\sum_{ i = 1}^{|\calD_{\attacker}^{(k)}|}||\nabla \calL(\theta, X_{t,i}^{(k)}, Y_i^{(k)}) - \nabla \calL(\theta, \wtilde X_{i}^{(k)}, Y_i^{(k)})||\nonumber\\
    & = c_a\cdot\Delta_{\text{up}}^{(k)} - c_b\cdot\frac{1}{T}\sum_{t = 1}^T\frac{1}{|\calD_{\attacker}^{(k)}|}\sum_{ i = 1}^{|\calD_{\attacker}^{(k)}|}||\nabla \calL(\theta, X_{t,i}^{(k)}, Y_i^{(k)}) - \nabla \calL(\theta, \wtilde X_{i}^{(k)}, Y_i^{(k)})||\label{eq: data_gap_terms}, 
\end{align}
where the second inequality is due to $||\wtilde X_{i}^{(k)} - X_{i}^{(k)}||\ge c_a ||\nabla \calL(\theta, \wtilde X_{i}^{(k)}, Y_i^{(k)}) - \nabla \calL(\theta, X_{i}^{(k)}, Y_i^{(k)})||$ and $||X_{t,i}^{(k)} - \wtilde X_{i}^{(k)}||\le c_b ||\nabla \calL(\theta, X_{t,i}^{(k)}, Y_i^{(k)}) - \nabla \calL(\theta, \wtilde X_{i}^{(k)}, Y_i^{(k)})||$, the third inequality is due to the triangle inequality ($\|a\| + \|b\|\ge\|a + b\|$), and the last equality is due to $\|\frac{1}{|\calD_{\attacker}^{(k)}|}\sum_{ i = 1}^{|\calD_{\attacker}^{(k)}|} (\nabla \calL(\theta, \wtilde X_{i}^{(k)}, Y_i^{(k)}) - \nabla \calL(\theta, X_{i}^{(k)}, Y_i^{(k)}))\| = \|\delta^{(k)}\| = \Delta_{\text{up}}^{(k)}$ from \pref{lem: relation_between_Delta_and_delta}. Now we bound the second term of \pref{eq: data_gap_terms}. We have that
\begin{align*}
    &c_a\cdot\Delta_{\text{up}}^{(k)} - c_b\cdot\frac{1}{T}\sum_{t = 1}^T\frac{1}{|\calD_{\attacker}^{(k)}|}\sum_{ i = 1}^{|\calD_{\attacker}^{(k)}|}||\nabla \calL(\theta, X_{t,i}^{(k)}, Y_i^{(k)}) - \nabla \calL(\theta, \wtilde X_{i}^{(k)}, Y_i^{(k)})||\\
    &\ge c_a\Delta_{\text{up}}^{(k)} - c_2\cdot c_b T^{-\frac{1}{2}}\\
    &\ge\frac{c_a}{2}\cdot\Delta_{\text{up}}^{(k)},
\end{align*}
where the second inequality is due to $c_a\Delta_{\text{up}}^{(k)}\ge 2 c_2\cdot c_b T^{-\frac{1}{2}}$ from our assumption.\\
Using Hoeffding's inequality (\pref{lem:hoeffdingbound}), we have that with probability at least $1 - \delta$,
\begin{align}
    \left|\frac{1}{|\calD_{\attacker}^{(k)}|} \sum_{i = 1}^{|\calD_{\attacker}^{(k)}|} X_i - \E\left[\frac{1}{|\calD_{\attacker}^{(k)}|} \sum_{i = 1}^{|\calD_{\attacker}^{(k)}|} X_i\right]\right|\le \sqrt{\frac{\ln(2/\gamma^{(k)})}{2|\calD_{\attacker}^{(k)}|}},
\end{align}
where $|\calD_{\attacker}^{(k)}|$ represents the size of the mini-batch (the training set of the attacker).\\
Therefore, with probability at least $1 - \delta$, we have
\begin{align*}
   &\frac{1}{|\calD_{\attacker}^{(k)}|} \sum_{i = 1}^{|\calD_{\attacker}^{(k)}|} \frac{1}{T}\sum_{t = 1}^T \frac{||X_{t,i}^{(k)} - X_{i}^{(k)}||}{D}\\
   &\le \E\left[\frac{1}{|\calD_{\attacker}^{(k)}|} \sum_{i = 1}^{|\calD_{\attacker}^{(k)}|} \frac{1}{T}\sum_{t = 1}^T \frac{||X_{t,i}^{(k)} - X_{i}^{(k)}||}{D}\right] + \sqrt{\frac{\ln(2/\gamma^{(k)})}{2|\calD_{\attacker}^{(k)}|}}\\
   & = 1 - \epsilon_p^{(k)} + \sqrt{\frac{\ln(2/\gamma^{(k)})}{2|\calD_{\attacker}^{(k)}|}}.
\end{align*}
Therefore, we have
\begin{align*}
     1 - \epsilon_p^{(k)} + \sqrt{\frac{\ln(2/\gamma^{(k)})}{2|\calD_{\attacker}^{(k)}|}} 
     & \ge \frac{1}{|\calD_{\attacker}^{(k)}|}\sum_{ i = 1}^{|\calD_{\attacker}^{(k)}|}\frac{1}{T}\sum_{t = 1}^T \frac{||X_{t,i}^{(k)} - X_{i}^{(k)}||}{D}\\
     &\ge\frac{c_a\Delta_{\text{up}}^{(k)}}{2D}.
\end{align*}
Therefore,
   \begin{align}
    \Delta_{\text{up}}^{(k)}\le \frac{2D}{c_a}\cdot (1 - \epsilon_p^{(k)}) + \frac{2D}{c_a}\cdot\sqrt{\frac{\ln(2/\gamma^{(k)})}{2|\calD_{\attacker}^{(k)}|}}. 
  \end{align}
Notice that 
\begin{align}
    |\calD_{\attacker}^{(k)}| = |\calD_{\protector}^{(k)}|.
\end{align}
Therefore, we have that
\begin{align}
    \frac{c_a}{2D}\cdot\Delta_{\text{up}}^{(k)}\le 1 - \epsilon_p^{(k)} + \sqrt{\frac{\ln(2/\gamma^{(k)})}{2|\calD_{\protector}^{(k)}|}}. 
\end{align}
Therefore, we have
\begin{align}
    \epsilon_p^{(k)}\le 1 + \sqrt{\frac{\ln(2/\gamma^{(k)})}{2|\calD_{\protector}^{(k)}|}} - \frac{c_a}{2D}\cdot\Delta_{\text{up}}^{(k)}.
\end{align}
\end{proof}



\section{Analysis for \pref{thm: exponential_lower_bound_mt}}\label{app: exponential_lower_bound} 


We provide an exponential lower bound for the sample complexity of the attacker. With a sublinear amount of distortion of the model parameter, at least an exponentially amount of samples are required for the attacker to achieve PAC learnability. The high-level ideas are
\begin{itemize}
    \item First, the risk of a classifier generated by any learning algorithm with sub-exponential sample complexity is at least $\Omega\left(\frac{1}{|\calD_{\attacker}^{(k)}|}\right)$, where $|\calD_{\attacker}^{(k)}|$ represents the number of samples. 
    \item Second, we show that the original risk turns into a rather large adversarial risk with a certain amount of distortion, which contradicts the assumption of PAC learnability. 
\end{itemize}

Now we are ready to provide the analysis for this theorem that is motivated by \cite{diochnos2019lower}.

Notice that the distance between any data from the original dataset and that from the inferred dataset is about $\Theta(\Delta^{(k)})$.

\begin{thm}\label{thm: relationship_between_model_distortion_and_data_distortion}
   Let $\Delta_{\text{up}}^{(k)}$ denote the \textit{uplink distortion extent} (\pref{defi: uplink_distortion_extent}), and $\Delta^{(k)} = \max_{i} ||\wtilde X^{(k)}_i - X^{(k)}_i||$. Then we have that,
   \begin{align}
       \Delta^{(k)}\ge c_a\Delta_{\text{up}}^{(k)}.
   \end{align}
\end{thm}

\begin{proof}
We denote $g^{(k)}$ as the original gradient of client $k$, and $\wtilde{g}^{(k)}$ as the distorted gradient uploaded from client $k$ to the server. Recall that $g^{(k)} = \frac{1}{|\calD_{\attacker}^{(k)}|}\sum_{ i = 1}^{|\calD_{\attacker}^{(k)}|}\nabla \calL(\theta, X_{i}^{(k)}, Y_i^{(k)})$, and $\wtilde g^{(k)} = \frac{1}{|\calD_{\attacker}^{(k)}|}\sum_{ i = 1}^{|\calD_{\attacker}^{(k)}|} \nabla \calL(\theta, \wtilde X_{i}^{(k)}, Y_i^{(k)})$. The distortion extent of client $k$ 
\begin{align}
   \Delta_{\text{up}}^{(k)} = \| \wtilde{g}^{(k)} - g^{(k)}\|.
\end{align}
From \pref{assump: two-sided Lipschitz}, we have that
\begin{align}
    ||\wtilde X^{(k)}_i - X^{(k)}_i||
    & \ge c_a ||\nabla \calL(\theta, \wtilde X^{(k)}_i, Y_i) - \nabla \calL(\theta, X^{(k)}_i, Y_i)||.
\end{align}
Therefore, we have
\begin{align*}
    \max_{i} ||\wtilde X^{(k)}_i - X^{(k)}_i|| &\ge \frac{1}{|\calD_{\attacker}^{(k)}|}\sum_{ i = 1}^{|\calD_{\attacker}^{(k)}|} c_a ||\nabla \calL(\theta, \wtilde X_i^{(k)}, Y_i) - \nabla \calL(\theta, X_i^{(k)}, Y_i)||\\
    &\ge c_a \big\|\frac{1}{|\calD_{\attacker}^{(k)}|}\sum_{ i = 1}^{|\calD_{\attacker}^{(k)}|} \nabla \calL(\theta, \wtilde X_{i}^{(k)}, Y_i^{(k)}) - \frac{1}{|\calD_{\attacker}^{(k)}|}\sum_{ i = 1}^{|\calD_{\attacker}^{(k)}|}\nabla \calL(\theta, X_{i}^{(k)}, Y_i^{(k)})\big\|\\
    & = c_a ||\wtilde g^{(k)} - g^{(k)}||,
\end{align*}
where the second inequality is due to the triangle inequality ($\|a\| + \|b\|\ge\|a + b\|$).\\
Therefore, 
\begin{align}
    \Delta^{(k)} = \max_{i} ||\wtilde X^{(k)}_i - X^{(k)}_i||\ge c_a\|\distmodel^{(k)} - \truemodel^{(k)}\| = c_a\Delta_{\text{up}}^{(k)}.
\end{align}
Therefore, we have that
\begin{align}
    \Delta^{(k)} \ge c_a\Delta_{\text{up}}^{(k)}. 
\end{align}

\end{proof}

\begin{defi}[Normal Levy Family]
   Let $d$ represent the dimension. A family of metric probability spaces with concentration function $\alpha_d(\cdot)$ is called a normal Levy family, if there exist $k_1, k_2$ satisfying that
   \begin{align}
       \alpha_d(c)\le k_1\cdot e^{-k_2\cdot c^2/d}.
   \end{align}
\end{defi}

The intuition of the following lemma is: if the number of samples is not large enough, with high constant probability, all the instances in the training set could not provide any information for the learner to distinguish classifier $c_1$ from classifier $c_2$. Then, with high constant probability, we can provide a lower bound for the original risk of the learner.


\begin{lem}[Lower Bound for Original Risk]\label{lem: lower_bound_for_loss_app}
Let \pref{assump: about_close_functions} hold. Assume $1 > \delta > 0.5$. Let $|\calD_{\attacker}^{(k)}| = 2^{O(\Delta^{(k)})}$, then with probability at least $1 - \delta$, we have that
\begin{align}
    \risk(h,c)\ge\frac{a_1(2\delta - 1)}{4a_2|\calD_{\attacker}^{(k)}|},
\end{align}
where $a_1, a_2 > 0$ represent two constants.
\end{lem}
\begin{proof}
Recall \pref{assump: about_close_functions} implies that there exist two classifiers $c_1, c_2$ (from $\calX$ to $\calY$) satisfying that $\Pr_{\distpara\sim\distdistribution}[c_1(\distpara)\neq c_2(\distpara)] = \Theta(\epsilon)$, for all $1\ge\epsilon\ge 2^{-\Theta(d)}$.

We define $\diff (c_1, c_2) = \{\distpara\in\calX| c_1(\distpara)\neq c_2(\distpara)\}$. Note that $|\calD_{\attacker}^{(k)}| = 2^{O(\Delta^{(k)})}$. From \pref{assump: about_close_functions} and $\frac{\gamma}{2|\calD_{\attacker}^{(k)}|}\ge 2^{-\Theta(\Delta^{(k)})}\ge 2^{-\Theta(d)}$, there exist two classifiers $c_1, c_2$ such that 
\begin{align}
    \frac{a_2\gamma}{|\calD_{\attacker}^{(k)}|} \ge\Pr_{\distpara\sim\distdistribution}[\distpara\in\diff(c_1, c_2)]\ge\frac{a_1\gamma}{2|\calD_{\attacker}^{(k)}|},
\end{align}
where $a_1, a_2 > 0$ represent two constants. 


    Assume the learner has a total of $|\calD_{\attacker}^{(k)}|$ i.i.d. samples, which is denoted as the training set $\calS$.
    Let $\calX = \{X_1, X_2, \cdots, X_{|\calD_{\attacker}^{(k)}|}\}$, $\calE_i = \{X_i\notin\diff (c_1, c_2)\}$, and $\calE = \{\forall X \in\calX| X\notin\diff (c_1, c_2)\}$. Therefore,
    \begin{align}
        \Pr[\bar\calE_i]\le\frac{a_2\gamma}{|\calD_{\attacker}^{(k)}|}.
    \end{align}
    Then, the probability that event $\calE$ holds is

    \begin{align}
        \Pr[\calE] & = \Pr[\calE_1\cap\cdots\cap\calE_{|\calD_{\attacker}^{(k)}|}]\\
        & = 1 - \Pr[\bar\calE_1\cup\cdots\cup\bar\calE_{|\calD_{\attacker}^{(k)}|}]\\
        & \ge 1 - \sum_i\Pr[\bar\calE_i]\\
        & \ge 1 - |\calD_{\attacker}^{(k)}|\cdot\frac{a_2\gamma}{|\calD_{\attacker}^{(k)}|} \\
        & = 1 - a_2\gamma. 
    \end{align}
Therefore,   
    \begin{align}
        \Pr[\calE]\ge 1 - a_2\gamma = 2(1 - \delta).
    \end{align}
    
    If all the samples in the training set do not belong to $\diff (c_1, c_2)$ (which corresponds to $\calE$), then the learner could not distinguish $c_1$ from $c_2$. Let $c$ be a random classifier that is sampled from $\{c_1, c_2\}$. With probability at least $1 - \delta$,

    \begin{align*}
       \risk(h,c) & = \Pr_{\distpara\sim\distdistribution}[h(\distpara)\neq c(\distpara)]\\
        &\ge\frac{1}{2}\cdot\Pr_{\distpara\sim\distdistribution}[\distpara\in\diff (c_1, c_2)]\\
        &\ge\frac{a_1 \gamma}{4|\calD_{\attacker}^{(k)}|}\\
        & = \frac{a_1(2\delta - 1)}{4a_2|\calD_{\attacker}^{(k)}|}.
    \end{align*}
\end{proof}


The following lemma measures the relationship between the risk and the adversarial risk using the distortion extent.
\begin{lem}[Relationship between Risk and Adversarial Risk]\label{lem: prepare_loss_to_AdvLoss_app}
   Let $\calB$ represent a Borel set. If the original risk $\Pr_{\distpara\sim\distdistribution}[h(\distpara)\neq c(\distpara)] > 1 - \inf\{\Pr_{\distpara\sim\distdistribution}[d(\distpara, \calB)\le\Delta_1^{(k)}]|\Pr_{\distpara\sim\distdistribution}[\distpara\in\calB]\ge\frac{1}{2}\}$, and the distortion extent is $\Delta_1^{(k)} + \Delta_2^{(k)}$, then the adversarial risk is larger than $1-\gamma$, where $\gamma > 1 - \inf\{\Pr_{\distpara\sim\distdistribution}[d(\distpara, \calB)\le\Delta_2^{(k)}]|\Pr_{\distpara\sim\distdistribution}[\distpara\in\calB]\ge\frac{1}{2}\}$.
\end{lem}
\begin{proof}
The risk of a classifier $h\in\calH$ is defined as
\begin{align}
        \risk(h, c) = \Pr_{\distpara\sim\distdistribution}[h(\distpara)\neq c(\distpara)]. 
\end{align}
We denote $\diff(h,c) = \{X\in\calX| h(\truepara)\neq c(\truepara)\}$, and $\diff(h,c, \Delta^{(k)}) = \{X\in\calX|d(\truepara,\diff(h,c))\le \Delta^{(k)}\}$. The adversarial risk is defined as 
\begin{align}
    \advrisk(h,c,\Delta^{(k)}) 
    & = \Pr_{\distpara\sim\distdistribution}[\exists \truepara \in\calX: d(\truepara, \distpara)\le \Delta^{(k)}, h(\truepara)\neq c(\truepara)]\\
    & = \Pr_{\distpara\sim\distdistribution}[d(\distpara, \diff(h,c))\le\Delta^{(k)}].
\end{align}
Assume that $\advrisk(h,c,\Delta_1^{(k)})\le \frac{1}{2}$, and denote $\calB_0 = \calX\setminus\diff(h,c, \Delta_1^{(k)})$. Then, it holds that
\begin{align*}
    \Pr_{\distpara\sim\distdistribution}[\distpara\in\calB_0] & = \Pr_{\distpara\sim\distdistribution}[d(\distpara, \diff(h,c))>\Delta_1^{(k)}]\\
    & = 1 - \Pr_{\distpara\sim\distdistribution}[\exists \truepara \in\calX: d(\truepara, \distpara)\le \Delta_1^{(k)}, h(\truepara)\neq c(\truepara)]\\
    & = 1 - \advrisk(h,c,\Delta_1^{(k)})\\
    &\ge \frac{1}{2}.
\end{align*}
We denote $\calB_1 = \{X\in\calX|d(\truepara,\calB_0)\le \Delta_1^{(k)}\}$. We know that
\begin{align}
   \Pr_{\distpara\sim\distdistribution}[\distpara\in\calB_1] & = \Pr_{\distpara\sim\distdistribution}[d(\distpara, \calB_0)\le\Delta_1^{(k)}]\\
    &\ge \inf\{\Pr_{\distpara\sim\distdistribution}[d(\distpara, \calB)\le\Delta_1^{(k)}]|\Pr_{\distpara\sim\distdistribution}[\distpara\in\calB]\ge\frac{1}{2}\}\\
    &> 1- \Pr_{\distpara\sim\distdistribution}[h(\distpara)\neq c(\distpara)]\\
    & = \Pr_{\distpara\sim\distdistribution}[h(\distpara) = c(\distpara)],
\end{align}
where the first inequality is due to $\Pr_{\distpara\sim\distdistribution}[\distpara\in\calB_0]\ge \frac{1}{2}$, and the second inequality is due to $\Pr_{\distpara\sim\distdistribution}[h(\distpara)\neq c(\distpara)]>1 - \inf\{\Pr_{\distpara\sim\distdistribution}[d(\distpara, \calB)\le\Delta_1^{(k)}]|\Pr_{\distpara\sim\distdistribution}[\distpara\in\calB]\ge\frac{1}{2}\}$ from our assumption.\\
Therefore, there exist $\what X\in\calX\setminus\diff(h,c, \Delta_1^{(k)})$ and $\wtilde X\in\diff(h,c)\cap\calB_1$ such that
\begin{align}
    d(\wtilde X, \hat X)\le \Delta_1^{(k)}.
\end{align}
$\wtilde X\in\diff(h,c)$ further implies that
\begin{align}
    \what X\in\diff(h,c, \Delta_1^{(k)}),
\end{align}
which leads to a contradiction.\\
Therefore, we have that
\begin{align}
    \advrisk(h,c,\Delta_1^{(k)}) 
    & = \Pr_{\distpara\sim\distdistribution}[d(\distpara, \diff(h,c))\le\Delta_1^{(k)}]\\
    & = \Pr_{\distpara\sim\distdistribution}[\distpara\in\diff(h,c, \Delta_1^{(k)})]\\
    &>\frac{1}{2}.
\end{align}
Let $\calB_2 = \{X\in\calX|d(\truepara,\diff(h,c, \Delta_1^{(k)}))\le \Delta_2^{(k)}\}$. From $\gamma > 1 - \inf\{\Pr_{\distpara\sim\distdistribution}[d(\distpara, \calB)\le\Delta_2^{(k)}]|\Pr_{\distpara\sim\distdistribution}[\distpara\in\calB]\ge\frac{1}{2}\}$, we have that
\begin{align}
    \Pr_{\distpara\sim\distdistribution}[\distpara\in\calB_2]&\ge \inf\{\Pr_{\distpara\sim\distdistribution}[d(\distpara, \calB)\le\Delta_2^{(k)}]|\Pr_{\distpara\sim\distdistribution}[\distpara\in\calB]\ge\frac{1}{2}\}\\
    & > 1 - \gamma. 
\end{align}
Therefore, we have that
\begin{align}
    \advrisk(h,c,\Delta_1^{(k)} + \Delta_2^{(k)}) > 1 - \gamma. 
\end{align}
\end{proof}

Assume the original risk is $\Omega\left(\frac{1}{|\calD_{\attacker}^{(k)}|}\right)$, and the distortion extent is $\Theta(\sqrt{\ln |\calD_{\attacker}^{(k)}|})$. The following lemma shows that the adversarial risk is larger than $1 - \gamma$.


\begin{lem}\label{lem: delta_and_lower_bound_for_attacking}
Let \pref{assump: normal_levy_faimily} hold. Let $\eta > 0$ denote the original risk. Let $\Delta^{(k)} = \frac{\sqrt{\ln(k_1/\eta) + \ln(k_1/\gamma)}}{\sqrt{k_2 d}}$ represent the distortion extent. Then the adversarial risk is larger than $1 - \gamma$.
\end{lem}
\begin{proof}
Let the system be a $(k_1, k_2)$-Levy family, and $d$ represent the dimension. Setting $\gamma$ and $\eta$ as positive constants. 
Let $\Delta_2^{(k)} = \sqrt{\frac{\ln(k_1/\gamma)}{k_2\cdot d}}$, and denote
\begin{align}\label{eq: delta_1_and_delta_2}
   \Delta_1^{(k)} = \Delta^{(k)} - \Delta_2^{(k)}.
\end{align} 
Note that $\Delta^{(k)} = \frac{\sqrt{\ln(k_1/\eta) + \ln(k_1/\gamma)}}{\sqrt{k_2 d}}$. Therefore, we have that
\begin{align}
    \Delta_1^{(k)} = \Delta^{(k)} - \Delta_2^{(k)} > \sqrt{\frac{\ln(k_1/\eta)}{k_2\cdot d}}.
\end{align}
Therefore, we have
\begin{align}\label{eq: Delta_1_lower_bounding_epsilon}
    \eta > k_1\cdot e^{-k_2\cdot (\Delta_1^{(k)})^{2}\cdot d}.
\end{align}
From the definition of $(k_1, k_2)$-normal Levy family, we know that the concentration function $\alpha(\cdot)$ satisfies that
\begin{align}\label{eq: exp_k2_lower_bound}
    \alpha(\Delta^{(k)})\le k_1\cdot e^{-k_2\cdot (\Delta^{(k)})^{2}\cdot d}.
\end{align}
Therefore, we have that
\begin{align}\label{eq: gamma_lower_bound}
   \gamma = k_1\cdot e^{-k_2\cdot (\Delta_2^{(k)})^2\cdot d} \ge \alpha(\Delta_2^{(k)}),
\end{align}
and
\begin{align}\label{eq: eta_lower_bound}
    \eta  = \risk(h,c) = \Pr_{\distpara\sim\distdistribution}[h(\distpara)\neq c(\distpara)] > k_1\cdot e^{-k_2\cdot (\Delta_1^{(k)})^{2}\cdot d}\ge \alpha(\Delta_1^{(k)}),
\end{align}
where the first inequality is due to \pref{eq: Delta_1_lower_bounding_epsilon}, and the second inequality is due to \pref{eq: exp_k2_lower_bound}.\\
From \pref{lem: prepare_loss_to_AdvLoss_app}, \pref{eq: delta_1_and_delta_2}, \pref{eq: gamma_lower_bound} and \pref{eq: eta_lower_bound}, we know that the adversarial risk
\begin{align}
    \advrisk(h,c,\Delta^{(k)}) > 1 - \gamma,
\end{align}
where $\Delta^{(k)} = \Delta_1^{(k)} + \Delta_2^{(k)}$.
\end{proof}

\begin{thm}[Lower bound for PAC Learnability of the Attackers]\label{thm: exponential_lower_bound_main_result_app}
    Let \pref{assump: two-sided Lipschitz}, \pref{assump: about_close_functions} and \pref{assump: normal_levy_faimily} hold. Assume $1 > \delta > 0.5$. Let $c_a$ be introduced in \pref{assump: two-sided Lipschitz}. For any $(\epsilon, \delta)$-PAC learning attacker, the number of samples associated with protector $k$ should be at least
   \begin{align}
       |\calD_{\attacker}^{(k)}|\ge \min\{2\delta - 1, 1-\epsilon\}\cdot 2^{c_a^2{\Delta_{\text{up}}^{(k)}}^2},
   \end{align}
   where $\Delta_{\text{up}}^{(k)}$ represents the \textit{uplink distortion extent} of client $k$.
\end{thm}

\begin{proof}
Let $|\calD_{\attacker}^{(k)}|$ represent the number of samples required by the attacker to achieve $(\epsilon, \delta)$-PAC learning in Phase 2, which guarantees that
\begin{align}
    \Pr[\advrisk(h,c,\Delta^{(k)})\ge\epsilon]\le\delta. 
\end{align}

Let \pref{assump: normal_levy_faimily} hold. Let $\eta$ denote the original risk, and $\Delta^{(k)}$ represent the distortion extent. 
Assume that $|\calD_{\attacker}^{(k)}|< (1-\epsilon)\cdot 2^{{\Delta_{\text{up}}^{(k)}}^2}$. From \pref{lem: lower_bound_for_loss_app}, with probability at least $1 - \delta$, we have 
\begin{align}
    \eta = \risk(h,c)\ge\frac{a_1(2\delta - 1)}{4a_2|\calD_{\attacker}^{(k)}|} = \Omega\left(\frac{1}{|\calD_{\attacker}^{(k)}|}\right).
\end{align}
Let $\Delta^{(k)}>\frac{\sqrt{\ln(k_1/\eta) + \ln(k_1/(1 - \epsilon))}}{\sqrt{k_2 d}}$. From \pref{lem: delta_and_lower_bound_for_attacking}, the adversarial risk is at least $\epsilon$ for any attacker with a distortion extent of $\Delta^{(k)}$, which contradicts the assumption that the learner achieves $(\epsilon, \delta)$-PAC learning.\\
From the above analysis, we know that
\begin{align}
    \Delta^{(k)} & \le\frac{\sqrt{\ln(k_1/\eta) + \ln(k_1/(1 - \epsilon))}}{\sqrt{k_2 d}}\\ 
    &= O\left(\sqrt{\ln \frac{|\calD_{\attacker}^{(k)}|}{\min\{2\delta - 1, 1-\epsilon\}}}\right).
\end{align}
Therefore, $|\calD_{\attacker}^{(k)}|\ge\min\{2\delta - 1, 1-\epsilon\}\cdot 2^{(\Delta^{(k)})^2}$ is required to guarantee $(\epsilon, \delta)$-PAC learning.
From \pref{thm: relationship_between_model_distortion_and_data_distortion}, we know that
   \begin{align}
       \Delta^{(k)} \ge c_a\Delta_{\text{up}}^{(k)}.
   \end{align}
Therefore, we have 
\begin{align}
       |\calD_{\attacker}^{(k)}|\ge \min\{2\delta - 1, 1-\epsilon\}\cdot 2^{c_a^2{\Delta_{\text{up}}^{(k)}}^2}.
\end{align}
\end{proof}

\section{Analysis for \pref{thm: relation_utility_loss}}
\label{appendix: analysis_for_the_relation_utility_loss}

From the Breteganolle-Huber-Carol inequality (\cite{wellner2013weak}), we have the following concentration inequality.

\begin{lem}\label{lem: concentratioS_inequality}
   Let $(|S_1|, \cdots, |S_m|)$ represent i.i.d. multinomial random variable with parameter $(\mu_1, \cdots, \mu_m)$. Then with probability at least $1-\delta$, we have that
   \begin{align}
    \Pr\left[\sum_{i = 1}^{m} \left|\frac{|S_i|}{\sum_{j = 1}^m |S_j|} - \mu_i\right|\ge\epsilon\right]\le 2^{m}\exp\left(\epsilon^2\frac{-\sum_{j = 1}^m |S_j|}{2}\right).
\end{align}
\end{lem}

In terms of the covering number, we refer to the following lemma.

\begin{lem}[\cite{papaspiliopoulos2020high}]
   Let $P$ be a polytope in $\R^d$ with $M$ vertices and whose diameter is bounded by $1$. Then $P$ can be covered by at most $M^{\frac{1}{\epsilon^2}}$ Euclidean balls of radius $\epsilon>0$.
\end{lem}

This lemma could be further generalized as follows.
\begin{lem}\label{lem: number_of_cover_set}
   Let $P$ be a polytope in $\R^d$ with $M$ vertices and whose diameter is bounded by $L$. Then $P$ can be covered by at most $M^{\frac{L}{\epsilon^2}}$ Euclidean balls of radius $\epsilon>0$.
\end{lem}

Now we provide a bound for the size of the cover set.
\begin{lem}\label{lem: relation_between_r_and_N_k}
   Let $d$ be the dimension of the model parameter, and $D$ be introduced in \pref{assump: two-sided Lipschitz}. For $\lambda^{(k)} > 0$, the covering number of $\lambda^{(k)}$-cover is at most $(2d)^{2D/(\lambda^{(k)})^{2} + 1}$.
\end{lem}
\begin{proof}
Assume that the distribution of the feature has compact support $\calX\subset\R^{d}$. Therefore, there exists a constant $D>0$, satisfying that $\|x_1 - x_2\|_2\le D$, $\forall x_1, x_2\in\calX$. Therefore $\calX_{\protector}^{(k)}$ can be covered by a polytope with $l_2$-diameter smaller than $2D$ and $2d$ vertices. Let $\calN(\lambda^{(k)}, \calX_{\protector}^{(k)}, \|\cdot\|_2)$ represent the cardinality of the smallest $\lambda^{(k)}$-cover. From \pref{lem: number_of_cover_set}, the cardinality of the smallest $\lambda^{(k)}$-cover (the covering number of $\lambda^{(k)}$-cover) is at most $\calN(\lambda^{(k)}, \calX_{\protector}^{(k)}, \|\cdot\|_2) \le (2d)^{2D/(\lambda^{(k)})^{2} + 1} = N^{(k)}$ (e.g., $\lambda^{(k)} = \Delta_{\text{two}}$).

\end{proof}

\begin{thm}[Utility Loss against Semi-honest Attackers]\label{thm: relation_utility_loss_and_delta_app}
    Let $\theta\in\R^d$ represent the model parameter. Assume that $\max_{Z\in\calD_{\protector}^{(k)}}|\calL^{(k)}(\theta + \delta , Z)|\le M^{(k)}$. We denote the empirical loss as $\calL^{(k)}_{\text{emp}}(\theta, Z) = \frac{1}{m_k}\sum_{i = 1}^{m_k} \calL^{(k)}(\theta, Z_i^{(k)})$, and denote the expected loss as $\calL^{(k)}_{\text{exp}}(\theta + \delta) = \E_{Z\sim P}[\calL^{(k)}(\theta + \delta , Z)]$. With probability at least $1 - \eta^{(k)}$,
   \begin{align}
       \epsilon_u^{(k)} &= |\calL^{(k)}_{\text{exp}}(\theta + \delta) - \calL^{(k)}_{\text{emp}}(\theta)|\nonumber\\
       &\le C\cdot\lambda^{(k)} + C\cdot\Delta_{\text{two}} + M^{(k)}\sqrt{\frac{2N^{(k)}\ln 2 + 2\ln(1/\eta^{(k)})}{|\calD_{\protector}^{(k)}|}},
   \end{align}
   where $\calD_{\protector}^{(k)} = \{Z_1^{(k)}, \cdots, Z_{m_k}^{(k)}\}$ represents the training set of client $k$, $\lambda^{(k)} > \Delta_{\text{two}}$ represents a constant, $M^{(k)}$ is a constant, $N^{(k)} = (2d)^{2D/(\lambda^{(k)})^{2} + 1}$ represents the covering number of $\lambda^{(k)}$-cover (Definition \ref{defi: r_cover}), $\Delta_{\text{two}}$ represents the \textit{two-way distortion extent} (Definition \ref{defi: two_way_distortion_extent}), and $\|\delta\| = \Delta_{\text{two}}$.
\end{thm}
\begin{proof}
Let $\calD_{\protector}^{(k)} = \{Z_1^{(k)}, \cdots, Z_{m_k}^{(k)}\}$, and $\calX_{\protector}^{(k)} = \{X_1^{(k)}, \cdots, X_{m_k}^{(k)}\}\subset\mathbb R^{d}$. Let $\calX_{\protector}^{(k)}$ be partitioned into a total of $N^{(k)}$ disjoint sets $C_1, C_2, \cdots, C_{N^{(k)}}$. We denote $S_i$ as the set of index of points of $\calX_{\protector}^{(k)}$ that belong to the set $C_i$, and $n$ as the size of $\calX_{\protector}^{(k)}$. We decompose $\E_{Z\sim P}[\calL^{(k)}(\theta + \delta , Z)]$ as $\sum_{i = 1}^{N^{(k)}} \E[\calL^{(k)}(\theta + \delta , Z)|Z\in C_i]\Pr[Z\in C_i]$.

We can construct a $\lambda^{(k)}$-cover of $(\calX_{\protector}^{(k)}, \|\cdot\|_2)$, which is a point set $\{\mathbf u\}\subset\calX_{\protector}^{(k)}$ such that for any $\mathbf v\in\calX_{\protector}^{(k)}$, there exists $\mathbf u$ satisfies $\|\mathbf u - \mathbf v\|_2\le\lambda^{(k)}$. From \pref{lem: relation_between_r_and_N_k}, the cardinality of the smallest $\lambda^{(k)}$-cover is at most $N^{(k)} = (2d)^{2D/(\lambda^{(k)})^{2} + 1}$ (e.g., $\lambda^{(k)} = \|\delta\|$). 

\begin{align}
    \epsilon_u^{(k)} & = |\calL^{(k)}_{\text{exp}}(\theta + \delta) - \calL^{(k)}_{\text{emp}}(\theta)|\nonumber\\
    & = |\E_{Z\sim P}[\calL^{(k)}(\theta + \delta , Z)] - \frac{1}{m_k}\sum_{i = 1}^{m_k} \calL^{(k)}(\theta, Z_i^{(k)})|\nonumber\\
    & = |\sum_{i = 1}^{N^{(k)}} \E[\calL^{(k)}(\theta + \delta , Z)|Z\in C_i]\Pr[Z\in C_i] - \frac{1}{m_k}\sum_{i = 1}^{m_k} \calL^{(k)}(\theta, Z_i^{(k)})|\nonumber\\
    & = |\sum_{i = 1}^{N^{(k)}} \E[\calL^{(k)}(\theta + \delta , Z)|Z\in C_i]\frac{|S_i|}{m_k} - \frac{1}{m_k}\sum_{i = 1}^{m_k} \calL^{(k)}(\theta, Z_i^{(k)})|\nonumber\\
    & \quad + |\sum_{i = 1}^{N^{(k)}} \E[\calL^{(k)}(\theta + \delta , Z)|Z\in C_i]\Pr[Z\in C_i] - \sum_{i = 1}^{N^{(k)}} \E[\calL^{(k)}(\theta + \delta , Z)|Z\in C_i]\frac{|S_i|}{m_k}|\nonumber\\
    & = |\frac{1}{m_k}\sum_{i = 1}^{N^{(k)}}\sum_{j\in S_i}\E[\calL^{(k)}(\theta + \delta , Z)|Z\in C_i] - \frac{1}{m_k}\sum_{i = 1}^{N^{(k)}}\sum_{j\in S_i} \calL^{(k)}(\theta, Z_j^{(k)})|\nonumber\\
    & \quad + |\sum_{i = 1}^{N^{(k)}} \E[\calL^{(k)}(\theta + \delta , Z)|Z\in C_i]\Pr[Z\in C_i] - \sum_{i = 1}^{N^{(k)}} \E[\calL^{(k)}(\theta + \delta , Z)|Z\in C_i]\frac{|S_i|}{m_k}|\label{eq: term1_and_term2}.
\end{align}
\textbf{Bounding the first term of \pref{eq: term1_and_term2}.}\\
Now we provide bounds for $|\frac{1}{m_k}\sum_{i = 1}^{N^{(k)}}\sum_{j\in S_i}\E[\calL^{(k)}(\theta + \delta , Z)|Z\in C_i] - \frac{1}{m_k}\sum_{i = 1}^{N^{(k)}}\sum_{j\in S_i} \calL^{(k)}(\theta, Z_j^{(k)})|$.

\begin{align*}
    &\left|\frac{1}{m_k}\sum_{i = 1}^{N^{(k)}}\sum_{j\in S_i}\E[\calL^{(k)}(\theta + \delta , Z)|Z\in C_i] - \frac{1}{m_k}\sum_{i = 1}^{N^{(k)}}\sum_{j\in S_i} \calL^{(k)}(\theta, Z_j^{(k)})\right|\\
    &=\left|\frac{1}{m_k}\sum_{i = 1}^{N^{(k)}}\sum_{j\in S_i}\left(\E[\calL^{(k)}(\theta + \delta , Z)|Z\in C_i] - \calL^{(k)}(\theta, Z_j^{(k)})\right)\right|\\ 
    &\le\frac{1}{m_k}\sum_{i = 1}^{N^{(k)}}\sum_{j\in S_i} \left|\E[\calL^{(k)}(\theta + \delta , Z)|Z\in C_i] - \calL^{(k)}(\theta + \delta , Z_j^{(k)}) + \calL^{(k)}(\theta + \delta , Z_j^{(k)}) - \calL^{(k)}(\theta, Z_j^{(k)})\right|\\
    &\le\frac{1}{m_k}\sum_{i = 1}^{N^{(k)}}\sum_{j\in S_i} \left|\E[\calL^{(k)}(\theta + \delta , Z)|Z\in C_i] - \calL^{(k)}(\theta + \delta , Z_j^{(k)})\right| + \frac{1}{m_k}\sum_{i = 1}^{N^{(k)}}\sum_{j\in S_i} \left|\calL^{(k)}(\theta + \delta , Z_j^{(k)}) - \calL^{(k)}(\theta, Z_j^{(k)})\right|. 
\end{align*}
For any $\theta\in\R^d$, $Z\in C_i, Z_j\in S_i$, we have
\begin{align*}
   | \calL^{(k)}(\theta + \delta , Z) - \calL^{(k)}(\theta + \delta , Z_j)|
   &\le C\|Z - Z_j\|\\
   &\le C\cdot\lambda^{(k)}, 
\end{align*}
where the first inequality is due to the Lipschitz's inequality, and the second inequality is due to the definition of the cover set.\\
Therefore, we have
\begin{align}
    |\E[\calL^{(k)}(\theta + \delta , Z)|Z\in C_i] - \calL^{(k)}(\theta + \delta , Z_j^{(k)})|&\le \max_{Z\in C_i} | \calL^{(k)}(\theta + \delta , Z) - \calL^{(k)}(\theta + \delta , Z_j^{(k)})|\nonumber\\
    &\le C\cdot\lambda^{(k)}\label{eq: loss_gap_z_and_Z_j^{(k)}}.
\end{align}
From Lipschitz's inequality, we know that
\begin{align}\label{eq: loss_gap_w_and_w_plus_delta}
    |\calL^{(k)}(\theta, Z_j^{(k)}) - \calL^{(k)}(\theta + \delta , Z_j^{(k)})|\le C\|\delta\|.
\end{align}
Therefore,
\begin{align*}
    &\left|\frac{1}{m_k}\sum_{i = 1}^{N^{(k)}}\sum_{j\in S_i}\E[\calL^{(k)}(\theta + \delta , Z)|Z\in C_i] - \frac{1}{m_k}\sum_{i = 1}^{N^{(k)}}\sum_{j\in S_i} \calL^{(k)}(\theta, Z_j^{(k)})\right|\\
     &\le\frac{1}{m_k}\sum_{i = 1}^{N^{(k)}}\sum_{j\in S_i} \left|\E[\calL^{(k)}(\theta + \delta , Z)|Z\in C_i] - \calL^{(k)}(\theta + \delta , Z_j^{(k)})\right|\\ 
     & + \frac{1}{m_k}\sum_{i = 1}^{N^{(k)}}\sum_{j\in S_i} \left|\calL^{(k)}(\theta + \delta , Z_j^{(k)}) - \calL^{(k)}(\theta, Z_j^{(k)})\right|\\
    &\le C\cdot\lambda^{(k)} + C\cdot\|\delta\| \\
    & = C\cdot\lambda^{(k)} + C\cdot \Delta_{\text{two}},
\end{align*}
where the second inequality is due to \pref{eq: loss_gap_z_and_Z_j^{(k)}} and \pref{eq: loss_gap_w_and_w_plus_delta}, and the last equality is due to $\Delta_{\text{two}} = \|\delta\|$ from \pref{lem: relation_between_two_link_distortion_and_delta}.\\
\textbf{Bounding the second term of \pref{eq: term1_and_term2}.}\\
Now we provide bounds for $|\sum_{i = 1}^{N^{(k)}} \E[\calL^{(k)}(\theta + \delta , Z)|Z\in C_i]\Pr[Z\in C_i] - \sum_{i = 1}^{N^{(k)}} \E[\calL^{(k)}(\theta + \delta , Z)|Z\in C_i]\frac{|S_i|}{m_k}|$.
We have
\begin{align*}
    &|\sum_{i = 1}^{N^{(k)}} \E[\calL^{(k)}(\theta + \delta , Z)|Z\in C_i]\Pr[Z\in C_i] - \sum_{i = 1}^{N^{(k)}} \E[\calL^{(k)}(\theta + \delta , Z)|Z\in C_i]\frac{|S_i|}{m_k}|\\
    &\le \max_{Z\in\calD_{\protector}^{(k)}}|\calL^{(k)}(\theta + \delta , Z)|\sum_{i = 1}^{N^{(k)}} \left|\Pr[Z\in C_i] - \frac{|S_i|}{m_k}\right|\\
    &\le M\sum_{i = 1}^{N^{(k)}}\left|\Pr[Z\in C_i] - \frac{|S_i|}{m_k}\right|.
\end{align*}
From \pref{lem: concentratioS_inequality}, with probability at least $1-\eta^{(k)}$, we have that
\begin{align}
    \sum_{i = 1}^{N^{(k)}} \left|\frac{|S_i|}{m_k} - \Pr[Z\in C_i]\right|\le\sqrt{\frac{2{N^{(k)}}\ln 2 + 2\ln(1/\eta^{(k)})}{m_k}}.
\end{align}
Therefore, we have
\begin{align}
    &|\calL^{(k)}_{\text{exp}}(\theta + \delta) - \calL^{(k)}_{\text{emp}}(\theta, Z)|\\ 
    & = |\frac{1}{m_k}\sum_{i = 1}^{N^{(k)}}\sum_{j\in S_i}\E[\calL^{(k)}(\theta + \delta , Z)|Z\in C_i] - \frac{1}{m_k}\sum_{i = 1}^{N^{(k)}}\sum_{j\in S_i} \calL^{(k)}(\theta, Z_j^{(k)})|\nonumber\\
    & \quad + |\sum_{i = 1}^{N^{(k)}} \E[\calL^{(k)}(\theta + \delta , Z)|Z\in C_i]\Pr[Z\in C_i] - \sum_{i = 1}^{N^{(k)}} \E[\calL^{(k)}(\theta + \delta , Z)|Z\in C_i]\frac{|S_i|}{m_k}|\nonumber\\
    &\le \frac{1}{m_k}\sum_{i = 1}^{N^{(k)}}\sum_{j\in S_i}\max_{Z_j^{(k)}\in C_i}|\calL^{(k)}(\theta, Z_j^{(k)}) - \calL^{(k)}(\theta + \delta , Z_j^{(k)})|\nonumber\\
    & \quad + \max_{Z\in\calD_{\protector}^{(k)}}|\calL^{(k)}(\theta + \delta , Z)|\sum_{i = 1}^{N^{(k)}} \left|\Pr[Z\in C_i] - \frac{|S_i|}{m_k}\right|\nonumber\\
    &\le C\cdot\lambda^{(k)} + C\cdot\frac{1}{K}\sum_{k = 1}^K \Delta_{\text{two}} + M^{(k)}\sum_{i = 1}^{N^{(k)}}\left|\Pr[Z\in C_i] - \frac{|S_i|}{m_k}\right|\nonumber\\
    &\le C\cdot\lambda^{(k)} + C\cdot\Delta_{\text{two}} + M^{(k)}\sqrt{\frac{2N^{(k)}\ln 2 + 2\ln(1/\eta^{(k)})}{|\calD_{\protector}^{(k)}|}},
\end{align}
where $N^{(k)} = (2d)^{2D/(\lambda^{(k)})^{2} + 1}$.
\end{proof}

\section{Analysis for \pref{thm: relation_between_two_link_distortion_and_delta_mt}}

The following lemma illustrates the relationship between $\delta$ and $\Delta_{\text{two}}$.
\begin{lem}\label{lem: relation_between_two_link_distortion_and_delta}
Let $\Delta_{\text{two}}$ denote the \textit{two-way distortion extent} (\pref{defi: two_way_distortion_extent}). The protector $k$ distorts the gradient according to \pref{eq: distortion_approach}, and $\delta = \wtilde\theta - \theta$. Assume that the protection mechanism satisfies that $\text{Dec}(\theta) = \theta$, $\forall\theta$. Then we have
   \begin{align}
          \Delta_{\text{two}} = \|\delta\|. 
      \end{align}
\end{lem}
\begin{proof}
      The process of sending updated parameter from clients to the server is referred to as uplink. The process of sending parameters from the server to clients is referred to as downlink. We denote $\theta$ as the original information, and $\wtilde{\theta}$ as the distorted information downloaded from the server and initialized for local optimization. 
   The \textit{two-way distortion extent} 
      \begin{align}
         \Delta_{\text{two}} 
         & = \| \text{Dec}(\wtilde\theta) - \theta\|\\
         & = \| \wtilde\theta - \theta\|\\
         & = \|\delta\|.
      \end{align}
\end{proof}

The following lemma illustrates that $\Delta_{\text{up}}^{(k)}$ is equal to the norm of $\delta^{(k)}$.
\begin{lem}\label{lem: relation_between_Delta_and_delta}
Let $\Delta_{\text{up}}^{(k)}$ denote the \textit{uplink distortion extent} (\pref{defi: uplink_distortion_extent}). The protector $k$ distorts the gradient according to \pref{eq: distortion_approach}, and $\delta^{(k)}$ is introduced therein. Then we have
   \begin{align}
       \Delta_{\text{up}}^{(k)} = \|\delta^{(k)}\|.
   \end{align}
\end{lem}
\begin{proof}
We denote $g^{(k)}$ as the original gradient of client $k$, and $\wtilde{g}^{(k)}$ as the distorted gradient uploaded from client $k$ to the server. The distortion extent of client $k$ 
\begin{align}
   \Delta_{\text{up}}^{(k)} = \| \wtilde{g}^{(k)} - g^{(k)}\|.
\end{align}
Recall that $g^{(k)} = \frac{1}{|\calD_{\attacker}^{(k)}|}\sum_{ i = 1}^{|\calD_{\attacker}^{(k)}|}\nabla \calL(\theta, X_{i}^{(k)}, Y_i^{(k)})$, and $\wtilde g^{(k)} = \frac{1}{|\calD_{\attacker}^{(k)}|}\sum_{ i = 1}^{|\calD_{\attacker}^{(k)}|} \nabla \calL(\theta, \wtilde X_{i}^{(k)}, Y_i^{(k)})$.
Therefore, we have
\begin{align}
   \Delta_{\text{up}}^{(k)} 
   & = \| \wtilde{g}^{(k)} - g^{(k)}\|\\
   & = \|\frac{1}{|\calD_{\attacker}^{(k)}|}\sum_{ i = 1}^{|\calD_{\attacker}^{(k)}|} (\nabla \calL(\theta, \wtilde X_{i}^{(k)}, Y_i^{(k)}) - \nabla \calL(\theta, X_{i}^{(k)}, Y_i^{(k)}))\|\\
   & = \|\frac{1}{|\calD_{\attacker}^{(k)}|}\sum_{ i = 1}^{|\calD_{\attacker}^{(k)}|}\nabla \calL(\theta, X_{i}^{(k)}, Y_i^{(k)}) + \delta^{(k)} - \nabla \calL(\theta, X_{i}^{(k)}, Y_i^{(k)}))\|\\
   & = \|\delta^{(k)}\|. 
\end{align}  
\end{proof}


\begin{lem}\label{lem: delta_two_and_delta_up_for_randomization}
Let $\Delta_{\text{two}}$ denote the \textit{two-way distortion extent} (\pref{defi: two_way_distortion_extent}), and $\Delta_{\text{up}}^{(k)}$ denote the \textit{uplink distortion extent} (\pref{defi: uplink_distortion_extent}). The protector $k$ distorts the gradient according to \pref{eq: distortion_approach}. Assume that $\text{Dec}(\wtilde\theta) = \wtilde\theta$, and that $\delta^{(1)} = \cdots = \delta^{(K)}$, where $\delta^{(k)}$ are introduced in \pref{eq: distortion_approach}. Then we have
\begin{align}
    \Delta_{\text{two}} = \Delta_{\text{up}}^{(k)}.
\end{align}
\end{lem}
\begin{proof}
   The process of sending updated parameter from clients to the server is referred to as uplink. The process of sending parameters from the server to clients is referred to as downlink. We denote $\theta$ as the original information, and $\wtilde\theta$ as the distorted information downloaded from the server and initialized for local optimization. 
   The \textit{two-way distortion extent} 
      \begin{align}
         \Delta_{\text{two}} 
         & = \| \text{Dec}(\wtilde\theta) - \theta\|\\
         & = \|\wtilde\theta - \theta\|
      \end{align}
      measures the gap between the original information and the distorted information during the whole process of uplink and downlink, and is related to utility loss.
      
We denote $\wtilde g^{(k)} = \frac{1}{|\calD_{\protector}^{(k)}|}\sum_{i = 1}^{|\calD_{\protector}^{(k)}|} \nabla \calL(\theta, X_{i}^{(k)}, Y_i^{(k)}) + \delta^{(k)}$, and $g^{(k)} = \frac{1}{|\calD_{\protector}^{(k)}|}\sum_{i = 1}^{|\calD_{\protector}^{(k)}|} \nabla \calL(\theta, X_{i}^{(k)}, Y_i^{(k)})$. The protected model parameter

  \begin{align}
      \wtilde\theta 
      & \leftarrow \theta - \eta\cdot\sum_{k = 1}^K \frac{|\calD^{(k)}|}{\sum_{k = 1}^K |\calD^{(k)}|}\wtilde g^{(k)}\\
      & = \theta - \eta\cdot\sum_{k = 1}^K \frac{|\calD^{(k)}|}{\sum_{k = 1}^K |\calD^{(k)}|} g^{(k)} - \eta\cdot\sum_{k = 1}^K \frac{|\calD^{(k)}|}{\sum_{k = 1}^K |\calD^{(k)}|}\delta^{(k)}.
  \end{align}
The unprotected model parameter 
  \begin{align}
      \theta\leftarrow \theta - \eta\cdot\sum_{k = 1}^K \frac{|\calD^{(k)}|}{\sum_{k = 1}^K |\calD^{(k)}|} g^{(k)}.
  \end{align}
Therefore, we have
\begin{align}
    \Delta_{\text{two}} & = \|\wtilde\theta - \theta\|\\
     & = \|\delta\|\\
     & = \|- \eta\cdot\sum_{k = 1}^K \frac{|\calD^{(k)}|}{\sum_{k = 1}^K |\calD^{(k)}|}\delta^{(k)}\|.
\end{align}

We denote $g^{(k)}$ as the original gradient, and $\wtilde{g}^{(k)}$ as the distorted gradient uploaded from the client to the server. The original model of client $k$
\begin{align}
    \theta^{(k)} = \theta -\eta\cdot\frac{1}{|\calD_{\protector}^{(k)}|}\sum_{i = 1}^{|\calD_{\protector}^{(k)}|} \nabla \calL(\theta, X_{i}^{(k)}, Y_i^{(k)}) = \theta - \eta\cdot g^{(k)}, 
\end{align}
and the distorted model of client $k$
\begin{align}
    \wtilde\theta^{(k)} = \theta - \eta\cdot \wtilde g^{(k)}. 
\end{align}
The distortion extent of client $k$ 
\begin{align}
   \Delta_{\text{up}}^{(k)} 
   & = \| \wtilde{\theta}^{(k)} - \theta^{(k)}\|\\
   & = \| \eta\cdot(\wtilde{g}^{(k)} - g^{(k)})\|\\
   & = \|\eta\cdot\delta^{(k)}\|,
\end{align}
where the third equality is due to \pref{lem: relation_between_Delta_and_delta}.
Assume $\delta^{(1)} = \cdots = \delta^{(K)}$. Then we have
\begin{align}
    \Delta_{\text{two}} = \Delta_{\text{up}}^{(k)}.
\end{align}
\end{proof}

   



\section{Analysis for \pref{thm: trade-off analysis for general protection mechanism}}\label{sec: trade-off analysis for general protection mechanism}
\begin{thm}[Utility, Privacy and Efficiency Trade-off for General Protection Mechanisms]
   Let $M^{(k)}$ be a constant satisfying that $\max_{Z\in\calD_{\protector}^{(k)}}|\calL(\theta + \delta, Z)|\le M^{(k)}$. Assume that $\Delta_{\text{two}} = \Theta(\Delta_{\text{up}}^{(k)})$, then there exists a constant $L$, satisfying that $\Delta_{\text{two}}\le L\cdot\Delta_{\text{up}}^{(k)}$. With probability at least $1 - \eta^{(k)} - \sum_{k =1}^K \gamma^{(k)}$, we have that
   \begin{align}
      \epsilon_u^{(k)} 
       &\le C\cdot L\cdot\frac{1}{K}\sum_{k = 1}^K \left(\frac{2D}{c_a}\cdot (1 - \epsilon_p^{(k)}) + \frac{2D}{c_a}\cdot\sqrt{\frac{\ln(1/\gamma^{(k)})}{\epsilon_e^{(k)}}}\right)\nonumber\\
       & + M^{(k)}\sqrt{\frac{2 N^{(k)}\ln 2 + 2\ln(1/\eta^{(k)})}{\epsilon_e^{(k)}}}.
   \end{align}
\end{thm}
\begin{proof}
   From \pref{thm: relation_utility_loss}, we have that
      \begin{align}
       \epsilon_u^{(k)}\le C\cdot\frac{1}{K}\sum_{k = 1}^K \Delta_{\text{two}} + M^{(k)}\sqrt{\frac{2 N^{(k)}\ln 2 + 2\ln(1/\eta^{(k)})}{\epsilon_e^{(k)}}},
   \end{align}
   From \pref{thm: privacy_distortion_and_datasize}, we have that 
      \begin{align}
    \Delta_{\text{up}}^{(k)}\le \frac{2D}{c_a}\cdot (1 - \epsilon_p^{(k)}) + \frac{2D}{c_a}\cdot\sqrt{\frac{\ln(1/\gamma^{(k)})}{\epsilon_e^{(k)}}}. 
  \end{align}
     From the assumption that $\Delta_{\text{two}}\le L\cdot\Delta_{\text{up}}^{(k)}$, we have that
   \begin{align}
      \epsilon_u^{(k)} 
       &\le C\cdot L\cdot\frac{1}{K}\sum_{k = 1}^K \left(\frac{2D}{c_a}\cdot (1 - \epsilon_p^{(k)}) + \frac{2D}{c_a}\cdot\sqrt{\frac{\ln(1/\gamma^{(k)})}{\epsilon_e^{(k)}}}\right)\nonumber\\
       & + M^{(k)}\sqrt{\frac{2 N^{(k)}\ln 2 + 2\ln(1/\eta^{(k)})}{\epsilon_e^{(k)}}},
   \end{align}
   where $L$ represents a constant.
\end{proof}

\section{Analysis for \pref{thm: utility_privacy_efficiency}}\label{appendix: analysis_for_utility_privacy_efficiency}

\begin{thm}[Trade-off Between Utility, Privacy and Efficiency for \textit{Randomization Mechanism}]\label{thm: utility_privacy_efficiency_app}
   Assume that $\max_{Z\in\calD_{\protector}^{(k)}}|\calL(\theta + \delta, Z)|\le M^{(k)}$. For \textit{Randomization Mechanism}, with probability at least $1 - \eta^{(k)} - \sum_{k =1}^K \gamma^{(k)}$, we have that
   \begin{align}\label{eq: utility_privacy_efficiency}
       \epsilon_u^{(k)} 
       &\le (2 + \rho) C\cdot\frac{1}{K}\sum_{k = 1}^K \frac{2D}{c_a}\cdot (1 - \epsilon_p^{(k)}) + (2 + \rho) C\cdot\frac{1}{K}\sum_{k = 1}^K \frac{2D}{c_a}\cdot\sqrt{\frac{\ln(1/\gamma^{(k)})}{\epsilon_e^{(k)}}}\nonumber\\
       & + M^{(k)}\sqrt{\frac{2 N^{(k)}\ln 2 + 2\ln(1/\eta^{(k)})}{\epsilon_e^{(k)}}},
   \end{align}
   where $\rho > 0$ represents a constant.
\end{thm}
\begin{proof}
   From \pref{thm: relation_utility_loss}, we have that
   \begin{align}
       \epsilon_u^{(k)}\le C\cdot\lambda^{(k)} + C\cdot\Delta_{\text{two}} + M^{(k)}\sqrt{\frac{2N^{(k)}\ln 2 + 2\ln(1/\eta^{(k)})}{\epsilon_e^{(k)}}},
   \end{align}
   where $\lambda^{(k)}> \Delta_{\text{two}}$ represents a constant. Setting $\lambda^{(k)} = (1 + \rho)\cdot\Delta_{\text{two}}$, where $\rho > 0$ represents a constant. We have
   \begin{align}
      \epsilon_u^{(k)} 
      &\le C\cdot\lambda^{(k)} + C\cdot\Delta_{\text{two}} + M^{(k)}\sqrt{\frac{2N^{(k)}\ln 2 + 2\ln(1/\eta^{(k)})}{\epsilon_e^{(k)}}}\\
      &= (2 + \rho) C\cdot\Delta_{\text{two}} + M^{(k)}\sqrt{\frac{2N^{(k)}\ln 2 + 2\ln(1/\eta^{(k)})}{\epsilon_e^{(k)}}}. 
   \end{align}
From \pref{thm: privacy_distortion_and_datasize}, we have that 
      \begin{align}
    \Delta_{\text{up}}^{(k)}\le \frac{2D}{c_a}\cdot (1 - \epsilon_p^{(k)}) + \frac{2D}{c_a}\cdot\sqrt{\frac{\ln(1/\gamma^{(k)})}{\epsilon_e^{(k)}}}. 
  \end{align}
     From \pref{lem: delta_two_and_delta_up_for_randomization}, for \textit{Randomization Mechanism}, we know that
   \begin{align}
      \Delta_{\text{two}} = \Delta_{\text{up}}^{(k)}.
   \end{align}
   Therefore, we have that
   \begin{align}
      \epsilon_u^{(k)} 
       &\le (2 + \rho) C\cdot\frac{1}{K}\sum_{k = 1}^K \frac{2D}{c_a}\cdot (1 - \epsilon_p^{(k)}) + (2 + \rho) C\cdot\frac{1}{K}\sum_{k = 1}^K \frac{2D}{c_a}\cdot\sqrt{\frac{\ln(1/\gamma^{(k)})}{\epsilon_e^{(k)}}}\nonumber\\
       & + M^{(k)}\sqrt{\frac{2 N^{(k)}\ln 2 + 2\ln(1/\eta^{(k)})}{\epsilon_e^{(k)}}}.
   \end{align}
\end{proof}

\section{Analysis for \pref{lem: two_way_distortion_HE_mt}}\label{appendix: two_way_distortion_HE_app}
The following lemma illustrates the value of $\Delta_{\text{two}}$ for \textit{Homomorphic Encryption Mechanism}.
\begin{lem}
   Let $\Delta_{\text{two}}$ denote the \textit{two-way distortion extent} (\pref{defi: two_way_distortion_extent}). For \textit{Homomorphic Encryption Mechanism}, we have that
   \begin{align}
          \Delta_{\text{two}} = 0. 
      \end{align}
\end{lem}
\begin{proof}
      The process of sending updated parameter from clients to the server is referred to as uplink. The process of sending parameters from the server to clients is referred to as downlink. We denote $\theta$ as the original information, and $\wtilde\theta$ as the distorted information downloaded from the server and initialized for local optimization. 
   The \textit{two-way distortion extent} 
      \begin{align}
         \Delta_{\text{two}} 
         & = \| \text{Dec}(\wtilde\theta) - \theta\|\\
         & = \|\theta - \theta\|\\
         & = 0.
      \end{align}
    Therefore, for \textit{Homomorphic Encryption Mechanism}, the \textit{two-way distortion extent}
      \begin{align}
          \Delta_{\text{two}} = 0. 
      \end{align}
\end{proof}

\section{Analysis for \pref{corr: utility_privacy_efficiency}}
\label{appendix: analysis_for_corollary_utility_privacy_efficiency}

\begin{corollary}[Private PAC Learning]\label{corr: utility_privacy_efficiency_app}
Let $\calX$ represent the feature set, and $\calY$ represent the label set. Let $\epsilon_p$ represent the amount of privacy leaked to the optimization-base semi-honest attackers.  
    For any $\alpha > c_2 (1 - \epsilon_p)$, given $n = \Theta(\frac{\ln(\frac{1}{\eta})}{(\alpha - c_2 (1 - \epsilon_p))^2})$ i.i.d. samples from any distribution $P$ on $\calX\times\calY$, with probability at least $1 - \eta$, the utility loss of the protector is at most $\alpha$, as is shown in the following equation:
    \begin{equation}
        {\rm Pr}(\epsilon_u \leq \alpha) \geq 1-\eta.
    \end{equation}
\end{corollary}
\begin{proof}
   From \pref{eq: utility_privacy_efficiency}, with probability at least $1 - \eta^{(k)} - \sum_{k =1}^K \gamma^{(k)}$, we have that
   \begin{align}
       \epsilon_u^{(k)} 
       &\le (2 + \rho) C\cdot\frac{1}{K}\sum_{k = 1}^K \frac{2D}{c_a}\cdot (1 - \epsilon_p^{(k)}) + (2 + \rho) C\cdot\frac{1}{K}\sum_{k = 1}^K \frac{2D}{c_a}\cdot\sqrt{\frac{\ln(1/\gamma^{(k)})}{|D^{(k)}_{\protector}|}}\nonumber\\
       & + M^{(k)}\sqrt{\frac{2 N^{(k)}\ln 2 + 2\ln(1/\eta^{(k)})}{|D^{(k)}_{\protector}|}},
   \end{align}
   where $\rho > 0$ represents a constant.
   
   Assume that $K = 1$, then with probability at least $1 - 2\eta$, 
   \begin{align}\label{eq: upper_bound_for_utility_loss_of_client_k}
       \epsilon_u\le (2 + \rho) C\cdot\frac{2D}{c_a}\cdot (1 - \epsilon_p) + (2 + \rho) C\cdot\frac{2D}{c_a}\cdot\sqrt{\frac{\ln(1/\gamma)}{|\calD_{\protector}|}} + M\sqrt{\frac{2 N\ln 2 + 2\ln(1/\eta)}{|\calD_{\protector}|}}. 
   \end{align}
Assume that
\begin{align}
    |\calD_{\protector}| = \Theta\left(\frac{\ln(\frac{1}{\eta})}{(\alpha - c_2 (1 - \epsilon_p))^2}\right),
\end{align}
then 
\begin{align}
    (2 + \rho) C\cdot\frac{2D}{c_a}\cdot (1 - \epsilon_p) + (2 + \rho) C\cdot\frac{2D}{c_a}\cdot\sqrt{\frac{\ln(1/\gamma)}{|\calD_{\protector}|}} + M\sqrt{\frac{2 N\ln 2 + 2\ln(1/\eta)}{|\calD_{\protector}|}} = \alpha.
\end{align}
Therefore, we have that
\begin{align}
    \epsilon_u\le\alpha.
\end{align}
\end{proof}

\end{document}